\documentclass{article}

    \PassOptionsToPackage{numbers}{natbib}

\usepackage[final]{neurips_2021}

\usepackage[utf8]{inputenc} 
\usepackage[T1]{fontenc}    
\usepackage{hyperref}       
\usepackage{nicefrac}       

\usepackage{cite}
\usepackage{amsmath,amssymb,amsfonts,amsthm}
\usepackage{bm}
\usepackage{algorithmic}
\usepackage{graphicx}
\usepackage{textcomp}
\usepackage{xcolor}
\usepackage{bbm}
\usepackage{dsfont}
\usepackage{subcaption}
\usepackage{url}
\usepackage{hyperref}
\usepackage{cleveref}
\usepackage[font=small]{caption}
\usepackage[ruled,vlined]{algorithm2e}
\usepackage{booktabs,makecell,multirow,hhline,threeparttable} 
\usepackage{xcolor,colortbl,pgf}
\definecolor{Color1}{RGB}{240, 240, 240}
\usepackage{makecell}
\usepackage{xspace}
\usepackage{units}
\usepackage{verbatim}
\usepackage{soul}
\usepackage{textcomp} 

\newcommand{\hb}{HBaR\xspace}
\newcommand{\nol}{M}
\newcommand{\reals}{\mathbb{R}}

\newcommand{\loss}{\mathcal{L}}
\newcommand{\aloss}{\loss_r}
\newcommand{\hbarloss}{\tilde{\loss}}
\newcommand{\hbaraloss}{\tilde{\loss}_r}
\newcommand{\HSIC}{\mathop{\mathrm{HSIC}}}
\newcommand{\dx}{d_X}
\newcommand{\dy}{k}
\newcommand{\dz}{d_Z}
\newcommand{\dzj}{d_{Z_j}}
\newcommand{\com}[1]{\textcolor{black}{#1}}

\newcommand{\fcom}[1]{\textcolor{black}{#1}}
\newcommand{\eqcom}{\color{black}}

\newtheorem{theorem}{Theorem}

\newtheorem{lemma}[theorem]{Lemma}
\newtheorem{assumption}{Assumption}

\newcommand{\E}{\mathbb{E}}

\newenvironment{packeditemize}{\begin{list}{$\bullet$}{\setlength{\itemsep}{0pt}\addtolength{\labelwidth}{1pt}\setlength{\leftmargin}{\labelwidth}\setlength{\listparindent}{\parindent}\setlength{\parsep}{1pt}\setlength{\topsep}{3pt}}}{\end{list}}

\title{Revisiting Hilbert-Schmidt Information Bottleneck for Adversarial Robustness}

\author{%
  Zifeng Wang\thanks{Equal contribution.}\\
  Northeastern University\\
  \texttt{zifengwang@ece.neu.edu} \\
  \And
  Tong Jian$^*$\\
  Northeastern University\\
  \texttt{jian@ece.neu.edu} \\
  \AND
  Aria Masoomi\\
  Northeastern University\\
  \texttt{masoomi.a@northeastern.edu} \\
   \And
  Stratis Ioannidis\\
  Northeastern University\\
  \texttt{ioannidis@ece.neu.edu} \\
  \And
  Jennifer Dy\\
  Northeastern University\\
  \texttt{jdy@ece.neu.edu} \\
}

\begin{document}

\maketitle

\begin{abstract}
  We investigate the HSIC (Hilbert-Schmidt independence criterion) bottleneck as a regularizer for learning an adversarially robust deep neural network classifier. \fcom{ In addition to the usual cross-entropy loss, we add regularization terms for every intermediate layer to ensure that the latent representations retain useful information for output prediction while reducing redundant information.} We show that the HSIC bottleneck enhances robustness to adversarial attacks both theoretically and experimentally. \fcom{In particular, we prove that the HSIC bottleneck regularizer reduces the sensitivity of the classifier to adversarial examples.} Our experiments on multiple benchmark datasets and architectures demonstrate that incorporating an HSIC bottleneck regularizer attains competitive natural accuracy and improves adversarial robustness, both with and without adversarial examples during training. \fcom{Our code and adversarially robust models are publicly available.\footnote{\texttt{\url{https://github.com/neu-spiral/HBaR}}}}
\end{abstract}

\section{Introduction} \label{sec:intro}
Adversarial attacks \citep{goodfellow2014explaining, madry2017towards, moosavi2016deepfool, carlini2017towards,autoattack} to deep neural networks (DNNs) have received considerable attention recently. Such attacks are intentionally crafted to change  prediction outcomes, e.g, by adding  visually imperceptible perturbations to the original, natural examples \citep{szegedy2013intriguing}. Adversarial robustness, i.e., the ability of a trained model to maintain its predictive power under such attacks, is an important property for many safety-critical applications \citep{adv_self_driving, adv_health, adv_surveillance}. The most common approach to construct adversarially robust models is via adversarial training \citep{yan2018deep, zhang2019theoretically, wang2019improving}, i.e., training the model over adversarially constructed samples.

%


Alemi et al.~\citep{alemi2016deep} propose using the so-called \emph{Information Bottleneck} (IB) \citep{tishby2000information, tishby2015deep} to ehnance adversarial robustness. Proposed by  Tishby and Zaslavsky \citep{tishby2015deep}, the information bottleneck expresses a tradeoff between (a) the mutual information of the input and latent  layers vs.~(b) the mutual information between latent layers and the output.  Alemi et al. show empirically that using IB as a learning objective for DNNs indeed leads to better adversarial robustness. Intuitively, the IB objective increases the entropy between input and latent layers; in turn, this also increases the model's robustness, as it makes latent layers less sensitive to  input perturbations.

Nevertheless, mutual information is notoriously expensive to compute. The  Hilbert-Schmidt independence criterion (HSIC) has been used as a tractable, efficient substitute in a variety of machine learning tasks \citep{knet, wu2020deep, wu2019solving}. Recently,  Ma et al.~\citep{ma2020hsic} also exploited this relationship to propose an \emph{HSIC bottleneck} (HB), as a variant to the more classic (mutual-information based) information bottleneck, though not in the context of adversarial robustness.

We revisit the HSIC bottleneck, studying its adversarial robustness properties. In contrast to both Alemi et al.~\citep{alemi2016deep} and Ma et al.~\citep{ma2020hsic}, we use the HSIC bottleneck as a regularizer in addition to commonly used losses for DNNs (e.g., cross-entropy). Our proposed approach, HSIC-Bottleneck-as-Regularizer (\hb) can be used in conjunction with adversarial 
examples; even without adversarial training, it is able to improve a classifier's robustness. It also significantly outperforms previous IB-based methods for robustness, as well as the method proposed by Ma et al. 

Overall, we make the following contributions:
\begin{packeditemize}
    \item[1.] We  apply the HSIC bottleneck as a regularizer for the purpose of adversarial robustness.
    \item[2.] We provide a theoretical motivation for  the constituent terms of the \hb penalty, proving that it indeed constrains the output perturbation produced by adversarial attacks. 
    \item[3.] We show that \hb can be naturally combined with a broad array of  state of the art adversarial training  methods, consistently improving their robustness.
    \item[4.] We empirically show that this phenomenon persists  even for weaker methods. In particular, \hb can  even enhance the adversarial robustness of plain SGD, without access to adversarial examples. 
\end{packeditemize}

The remainder of this paper is structured as follows. We  review  related work in Sec.~\ref{sec:related_work}. In Sec.~\ref{sec:background}, we  discuss the standard setting of adversarial robustness and HSIC. In Sec.~\ref{sec:method}, we provide a theoretical justification that HBaR reduces the sensitivity of the classifier to adversarial examples. Sec.~\ref{sec:experiments} includes our experiments; we conclude in Sec.~\ref{sec:conclusion}.

\begin{figure}[!t]
    \centering
    \includegraphics[width=0.7\columnwidth]{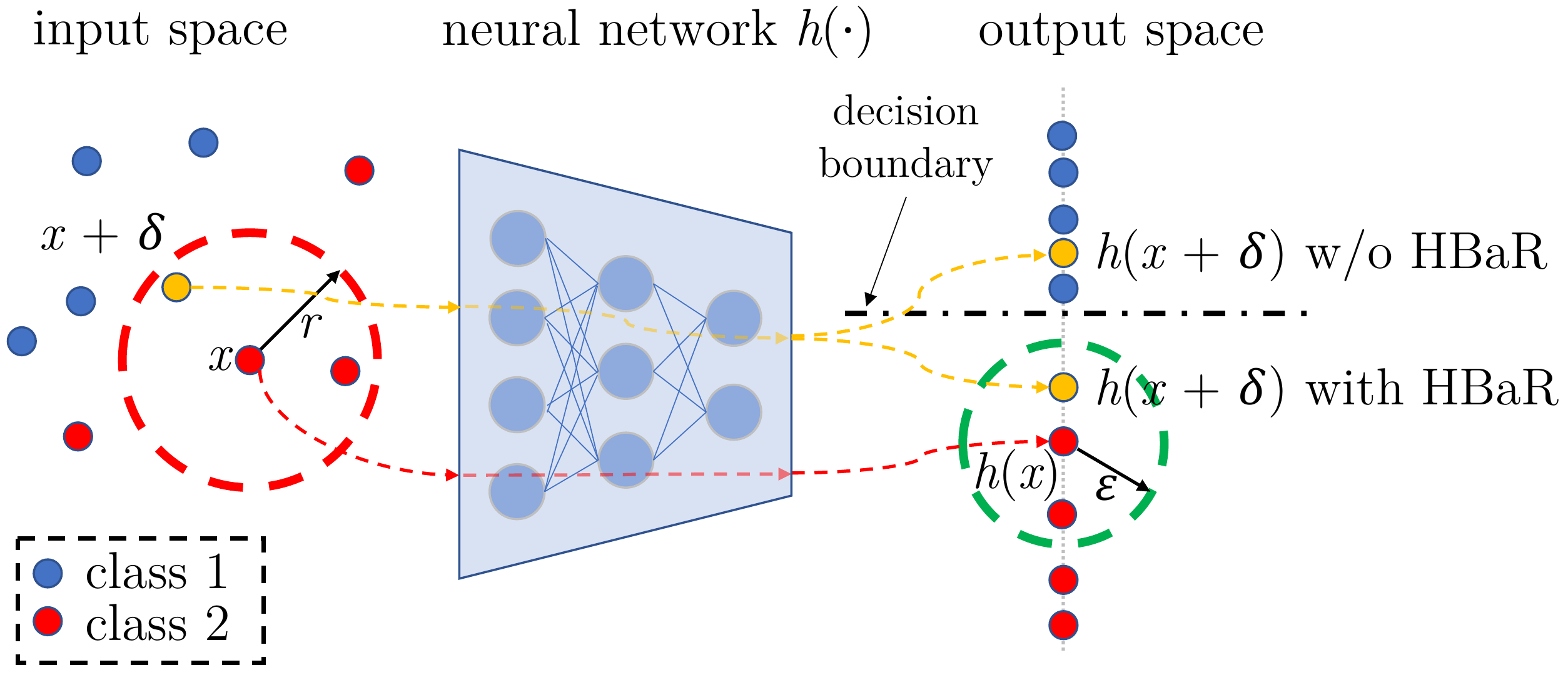}
    \caption{Illustration of \hb\ for adversarial robustness. A neural network trained with \hb gives a more constrained prediction w.r.t. perturbed inputs. Thus, it is less sensitive to adversarial examples.}
    \label{fig:hbar}
\end{figure}

\section{Related Work} \label{sec:related_work}
\textbf{Adversarial Attacks.} 
Adversarial attacks often add a constrained perturbation to  natural inputs with the goal of maximizing classification loss.  Szegedy et al.~\citep{szegedy2013intriguing} learn a perturbation  via box-constrained L-BFGS that misleads the classifier but minimally distort  the input. FGSM, proposed by Goodfellow et al. \citep{goodfellow2014explaining}, is a one step adversarial attack perturbing the input based on the sign of the gradient of the loss. PGD \citep{kurakin2016adversarial, madry2017towards}  generates adversarial examples through multi-step projected gradient descent optimization. DeepFool \citep{moosavi2016deepfool} is an iterative attack strategy, which perturbs the input towards the direction of the decision boundaries. CW  \citep{carlini2017towards} applies a rectifier function regularizer to generate adversarial examples near the original input. \com{ AutoAttack (AA) \citep{autoattack} is an ensemble of parameter-free attacks, that also deals with common issues like gradient masking \citep{gradient_masking} and fixed step sizes \citep{madry2017towards}.}

\textbf{Adversarial Robustness.}  A common approach to obtaining robust models is \textit{adversarial training}, i.e., training models over adversarial examples generated via the aforementioned attacks. For example, Madry et al.~\citep{madry2017towards} show that training with adversarial examples generated by PGD  achieves good robustness under different attacks.
DeepDefense \citep{yan2018deep}   penalizes the norm of adversarial perturbations. TRADES \citep{zhang2019theoretically} 
minimizes the difference between the predictions of natural and adversarial examples to get a smooth decision boundary. MART \citep{wang2019improving} pays more attention to adversarial examples from misclassified natural examples and adds a KL-divergence term between natural and adversarial samples to the cross-entropy loss.
%
\emph{We show that our proposed method \hb can be combined with several such state-of-the-art defense methods and  boost their performance.} 



\textbf{Information Bottleneck.} The information bottleneck (IB)  \citep{tishby2000information, tishby2015deep} expresses a tradeoff in  latent representations between   information useful for output prediction   and  information retained about the input. IB has been employed to explore the training dynamics in deep learning models \citep{shwartz2017opening, saxe2019information} as well  as a learning objective  \citep{alemi2016deep, amjad2018not}. \fcom{Fischer \citep{fischer2020conditional} proposes a conditional entropy bottleneck (CEB) based on IB and observes its robust generalization ability empirically.} Closer to us, Alemi et al.~\citep{alemi2016deep} propose a  variational information bottleneck (VIB) \fcom{for supervised learning}. They 
 empirically show that training VIB on natural examples  provides good generalization and adversarial robustness.
We show that \hb can be combined with  various adversarial defense methods enhancing their robustness, but also outperforms VIB \citep{alemi2016deep} when given access only to natural samples. Moreover, \emph{we provide theoretical guarantees on how \hb bounds the output perturbation induced by adversarial attacks.}

\noindent\textbf{Mutual Information vs.~HSIC.} Mutual information is difficult to compute in practice. To address this, Alemi et al.~\citep{alemi2016deep}   estimate IB via variational inference. Ma et al.~\citep{ma2020hsic} replaced mutual information by the Hilbert Schmidt Independence Criterion (HSIC) and named this the \emph{HSIC Bottleneck} (HB). 
Like Ma et al.~\citep{ma2020hsic}, we utilize HSIC to estimate IB.
However, our method is different from Ma et al.~\citep{ma2020hsic} in several aspects. First, they use HB to train the neural network stage-wise, layer-by-layer, without backpropagation, while we use HSIC bottleneck as a regularization in addition to cross-entropy and optimize the parameters jointly by backpropagation. Second, they only evaluate the model performance on classification accuracy, while we demonstrate adversarial robustness. 
Finally, we show that \hb further enhances robustness to adversarial examples both theoretically and experimentally. Greenfeld et al.~\citep{greenfeld2020robust}  use HSIC between the residual of the prediction and the input data as a learning objective for model robustness on covariate distribution shifts.  Their focus is on robustness to distribution shifts, whereas our work focuses on robustness to adversarial examples, on which \hb outperforms their proposed objective.


\section{Background} \label{sec:background}
\subsection{Adversarial Robustness}
In  standard $\dy$-ary classification, we are given a dataset $\mathcal{D} = \{(x_i, y_i)\}_{i= 1}^{n}$, where $x_i \in \reals^{\dx}, y_i \in \{0,1\}^{\dy}$ are i.i.d.~samples drawn from joint distribution $P_{XY}$. A learner trains a neural network $h_\theta:\reals^{\dx}\to\reals^{k}$ parameterized by weights $\theta\in \reals^{d_{\theta}}$ to predict $Y$ from $X$ by minimizing \begin{align}\loss(\theta)= \mathbb{E}_{XY}[\ell(h_{\theta}(X), Y)]\approx\frac{1}{n} \sum_{i=1}^n \ell(h_{\theta}(x_i), y_i)
, \label{eq:loss}
\end{align}
where $\ell:\reals^k \times \reals^k \to\reals$ is a loss function,  e.g., cross-entropy. 
We aim to find a model $h_\theta$ that has high prediction accuracy but is also \emph{adversarially robust}:  the model should maintain high prediction accuracy against a constrained adversary, that can perturb input samples in a restricted fashion. 
Formally, prior to submitting a sample $x\in \reals^{\dx}$ to the classifier, an adversary may perturb $x$ by an arbitrary $\delta\in \mathcal{S}_r$, where $\mathcal{S}_r\subseteq \reals^{\dx}$ is the $\ell_{\infty}$-ball of radius $r$, i.e., \begin{align} \label{def:S_r}
\mathcal{S}_r=B(0,r)=\{\delta \in  \reals^{\dx}: \|\delta\|_{\infty}\leq r\}.\end{align}
The \emph{adversarial robustness} \citep{madry2017towards} of a model $h_{\theta}$ is measured by the expected loss attained by such adversarial examples, i.e.,
\begin{equation}
    \label{eq:adv_robustness}
 \begin{split} \aloss(\theta)=  \mathbb{E}_{X Y} \left[\max_{\delta \in \mathcal{S}_r}\ell\left(h_{\theta}(X+\delta), Y\right)\right] \approx  \frac{1}{n} \sum_{i=1}^n \max_{\delta\in \mathcal{S}_r}\ell(h_{\theta}(x_i+\delta), y_i).\end{split}
\end{equation}

An adversarially robust neural network $h_\theta$ can be obtained via \emph{adversarial training}, i.e., by minimizing the adversarial robustness loss in \eqref{eq:adv_robustness} empirically over the training set $\mathcal{D}$. In practice, this amounts to training via stochastic gradient descent (SGD) over adversarial examples $x_i+\delta$ (see, e.g.,  \citep{madry2017towards}). In each epoch, $\delta$ is generated on a per sample basis via an inner optimization over $\mathcal{S}_r$, e.g., via projected gradient descent (PGD) on $-\loss$. 

\subsection{Hilbert-Schmidt Independence Criterion (HSIC)}
The Hilbert-Schmidt Independence Criterion (HSIC) is a statistical dependency measure introduced by Gretton et al.~\citep{gretton2005measuring}. HSIC is the Hilbert-Schmidt norm of the cross-covariance operator between the distributions in Reproducing Kernel Hilbert Space (RKHS). Similar to Mutual Information (MI), HSIC captures non-linear dependencies between random variables.  $\HSIC(X, Y)$ is defined as:
\fcom{
\begin{align}
\begin{split}
\HSIC(X, Y&) 
= \mathbb{E}_{X Y X^{\prime} Y^{\prime}}\left[k_{X}\left(X, X^{\prime}\right) k_{Y^{\prime}}\left(Y, Y^{\prime}\right)\right] \\
&+\mathbb{E}_{X X^{\prime}}\left[k_{X}\left(X, X^{\prime}\right)\right] \mathbb{E}_{YY^{\prime}}\left[k_{Y}\left(Y, Y^{\prime}\right)\right] \\
&-2 \mathbb{E}_{X Y}\left[\mathbb{E}_{X^{\prime}}\left[k_{X}\left(X, X^{\prime}\right)\right] \mathbb{E}_{Y^{\prime}}\left[k_{Y}\left(Y, Y^{\prime}\right)\right]\right],
\end{split}
\end{align}%
}%
where $X'$, $Y'$ are independent copies of $X$, $Y$, respectively, and $k_{X}$, $k_{Y}$ are kernels. 

In practice, we often approximate HSIC empirically. Given $n$ i.i.d.~samples $\{(x_i, y_i)\}_{i=1}^{n}$  drawn from $P_{XY}$, 
we estimate HSIC via:
\begin{equation}
    \label{eq:empirical_hsic}
    \widehat{\HSIC}(X, Y)={(n-1)^{-2}} \operatorname{tr}\left(K_{X} H K_{Y} H\right),
\end{equation}
where $K_X$ and $K_Y$ are kernel matrices with entries $K_{X_{ij}}=k_{X}(x_i, x_j)$ and $K_{Y_{ij}}=k_{Y}(y_i, y_j)$, respectively, and $H = \mathbf{I}- \frac{1}{n} \mathbf{1} \mathbf{1}^\top$ is a centering matrix.

\section{Methodology} \label{sec:method}
In this section, we present our method, HSIC bottleneck as regularizer (\hb) as a means to enhance a classifier's robustness. 
The effect of \hb\ for adversarial robustness is illustrated in Figure~\ref{fig:hbar}; the HSIC bottleneck penalty reduces the sensitivity of the classifier to adversarial examples. We provide a theoretical justification for this below, in Theorems~\ref{thm:main_theorem}~and~\ref{thm:new_theorem}, but also validate the efficacy of the HSIC bottleneck extensively with experiments in Section~\ref{sec:experiments}.

\subsection{HSIC Bottleneck as Regularizer for Robustness}
Given a feedforward neural network $h_\theta:\reals^{\dx}\to\reals^{k}$ parameterized by $\theta$ with $\nol$  layers, and an input r.v.~$X$, we denote by $Z_j\in\reals^{\dzj}$, $j\in \{1,\ldots,\nol\}$, the output of the $j$-th layer under input $X$ (i.e., the $j$-th latent representation). 
We define our \hb\  learning objective as follows:
\begin{align}
\label{eq:obj}
\begin{split}
\hbarloss(\theta) =  \loss(\theta)
    + \lambda_{x} &\sum_{j=1}^{\nol} \HSIC(X, Z_j) -\lambda_{y} \sum_{j=1}^{\nol} \HSIC(Y, Z_j), 
\end{split}
\end{align}
where $\loss$ is the standard loss given by Eq.~\eqref{eq:loss} and $\lambda_{x}, \lambda_{y}\in \reals_+$ are balancing hyperparameters. 

Together, the second and third terms in Eq.~\eqref{eq:obj} form the HSIC bottleneck penalty. As HSIC measures dependence between two random variables, minimizing $\HSIC(X, Z_i)$ corresponds to removing redundant or noisy information contained in $X$.
Hence, this term  also naturally reduces the influence of an adversarial attack, i.e., a perturbation added on the input data. This is intuitive, but we also provide theoretical justification in the next subsection.
Meanwhile, maximizing $\HSIC(Y, Z_i)$ encourages this lack of sensitivity to the input to happen while retaining the discriminative nature of the classifier, captured by dependence to useful information w.r.t. the output label $Y$. 
Note that minimizing $\HSIC(X, Z_i)$ alone would also lead to the loss of useful information, so it is necessary to keep the $\HSIC(Y, Z_i)$ term to make sure $Z_i$ is informative enough of $Y$. 


The overall algorithm is described in Alg.~\ref{alg:hsic}. In practice, we perform Stochastic Gradient Descent (SGD) over $\hbarloss$: both $\loss$ and HSIC can be evaluated empirically over batches. For the latter, we use the estimator \eqref{eq:empirical_hsic}, restricted over the current batch. 
As we have $m$ samples in a mini-batch, the complexity of calculating the empirical HSIC \eqref{eq:empirical_hsic} is $O(m^2d_{\bar{Z}})$ \citep{song2012feature} for a single layer, where $d_{\bar{Z}}=\max_{j}d_{Z_j}$. Thus, the overall complexity for~\eqref{eq:obj} is $O(Mm^2d_{\bar{Z}})$. This computation is highly parallelizable, thus, the additional computation time of \hb is small when compared to training a neural network via cross-entropy only.

\subsection{Combining \hb with Adversarial Examples}\label{sec:combine-hb}
\hb can also be naturally  applied in combination with adversarial training. For $r>0$ the magnitude of the perturbations introduced in adversarial examples, one can optimize the following objective instead of $\hbarloss(\theta)$ in Eq. \eqref{eq:obj}:
\begin{align}
\begin{split}    \hbaraloss(\theta) =  \aloss(\theta)
    + \lambda_{x} \sum_{j=1}^{\nol} \HSIC(X, Z_j) -\lambda_{y} \sum_{j=1}^{\nol} \HSIC(Y, Z_j),\end{split}
\end{align}
where $\aloss$ is the adversarial loss given by Eq.~\eqref{eq:adv_robustness}. This can be used instead of $\loss$ in Alg.~\ref{alg:hsic}. Adversarial examples need to be used in the computation of the gradient of the loss $\aloss$ in each minibatch; these need to be computed on a per sample basis, e.g., via PGD over $\mathcal{S}_r$, at additional computational cost. Note that the natural samples $(x_i,y_i)$ in a batch are used to compute the HSIC bottleneck regularizer.

The \hb penalty can similarly be combined with other adversarial learning methods and/or used with different means for selecting adversarial examples, other than PGD. We illustrate this in Section~\ref{sec:experiments}, where we combine \hb with state-of-the-art adversarial learning methods  TRADES \citep{zhang2019theoretically} and MART \citep{wang2019improving}. \\

\begin{algorithm}[!t]
\SetAlgoLined
\textbf{Input:} input sample tuples $\{(x_i, y_i)\}_{i=1}^{n}$, kernel function $k_x, k_y, k_z$, a neural network $h_\theta$ parameterized by $\theta$, mini-batch size $m$, learning rate $\alpha$.\\
\textbf{Output:} parameter of classifier $\theta$\\
 \While{$\theta$ has not converged}{
  Sample a mini-batch of size $m$ from input samples. \\
  Forward Propagation: calculate $z_i$ and $h_\theta(x)$.\\
  Compute kernel matrices for $X$, $Y$ and $Z_i$ using $k_x, k_y, k_z$ respectively inside mini-batch. \\
  Compute $\hbarloss(\theta)$ via \eqref{eq:obj}, where $\HSIC$ is evaluated empirically via \eqref{eq:empirical_hsic}.\\
  Backward Propagation: $\theta \leftarrow \theta - \alpha \nabla \hbarloss(\theta)$.
 }
 \caption{Robust Learning with HBaR}
 \label{alg:hsic}
\end{algorithm}

\subsection{\hb Robustness Guarantees} \label{sec:HBAR_theorem}
We provide  here a formal justification for the use of \hb  to enhance robustness: we prove that regularization terms $\HSIC(X, Z_j)$, $j=1,\ldots,\nol$  lead to classifiers which are less sensitive to input perturbations. 
For simplicity, we focus on the case where $k=1$ (i.e., binary classification). Let $Z\in \reals^{\dz}$ be the latent representation at some arbitrary intermediate layer of the network. That is, $Z=Z_j$, for some $j\in \{1,\ldots,\nol\}$; we omit the subscript $j$ to further reduce notation clutter. Then $h_\theta= (g \circ f)$, where $f: \reals^{\dx} \rightarrow \reals^{\dz}$ maps the inputs to this  intermediate layer, and $g: \reals^{\dz} \rightarrow \reals$ maps the intermediate layer to the final layer. Then, $Z = f(X)$ and $g(Z) = h_\theta(X)\in \reals$ are the latent and final outputs, respectively.
Recall that, in \hb, $\HSIC(X,Z)$ is associated with kernels $k_X$, $k_Z$. 
We make the following  technical assumptions:
\begin{assumption}\label{asm:cont}
Let $\mathcal{X}\subseteq \reals^{\dx}$, $\mathcal{Z}\subseteq \reals^{\dz} $ be the supports of random variables $X$, $Z$, respectively. We assume that both $h_\theta$ and $g$ are continuous and bounded functions in  $\mathcal{X}$, $\mathcal{Z}$, respectively, i.e.:
\begin{align}
    \label{eq:fi}
    h_\theta \in C(\mathcal{X}), 
    g \in C(\mathcal{Z}).
\end{align}
 Moreover, we assume that all functions $h_\theta$ and $g$ we consider are uniformly bounded, i.e., there exist $0 < M_\mathcal{X},M_\mathcal{Z} < \infty $ such that:
\begin{align} 
M_{\mathcal{X}} = \max_{h_\theta \in C(\mathcal{X})} \|h_\theta\|_\infty \quad\text{and} \quad M_{\mathcal{Z}} = \max_{g \in C(\mathcal{Z})} \|g\|_\infty. \label{eq:fbound} \end{align}
\end{assumption}
The continuity stated in Assumption~\ref{asm:cont} is natural, if all activation functions are continuous.  
Boundedness follows if, e.g., $\mathcal{X}$, $\mathcal{Z}$ are closed and bounded (i.e., compact), or if activation functions are bounded (e.g., softmax, sigmoid, etc.). 
\begin{assumption}\label{asm:universal}We assume kernels
$k_X$, $k_Z$ are universal with respect to functions $h_\theta$ and $g$ that satisfy Assumption \ref{asm:cont}, i.e., if $\mathcal{F}$ and $\mathcal{G}$ are the induced RKHSs for kernels $k_X$ and $k_Z$, respectively, then for any  $h_\theta,g$ that satisfy Assumption  \ref{asm:cont} 
and any $\varepsilon>0$ there exist functions $h'\in \mathcal{F}$ and $g'\in \mathcal{G}$ such that $||h_\theta-h'||_{\infty} \leq \varepsilon$ and $||g-g'||_{\infty} \leq \varepsilon$. Moreover, functions in $\mathcal{F}$ and $\mathcal{G}$ are uniformly bounded, i.e., there exist  $0<M_{\mathcal{F}}, M_{\mathcal{G}}<\infty $ such that for all $h'\in \mathcal{F}$ and all $g'\in \mathcal{G}$:
\begin{align}
M_{\mathcal{F}} = \max _{f'\in\mathcal{F}} \|f'\|_{\infty}\quad \text{and} \quad M_{\mathcal{G}} = \max _{g'\in\mathcal{G}} \|g'\|_{\infty}. \label{eq:kbound} 
\end{align}
\end{assumption}
We note that several kernels used in practice are universal, including, e.g., the Gaussian  and Laplace kernels.
Moreover, given that functions that satisfy Assumption~\ref{asm:cont} are uniformly bounded by \eqref{eq:fbound}, such kernels can indeed remain universal while  satisfying \eqref{eq:kbound} via an appropriate rescaling.
 
Our first   result shows that $\HSIC(X,Z)$ \emph{at any intermediate layer $Z$} bounds the \emph{output} variance:
\begin{theorem} \label{thm:main_theorem}     
Under Assumptions~\ref{asm:cont} and~\ref{asm:universal}, 
we have:
\begin{equation} \label{eq:main_theorem_v2}
    \operatorname{HSIC}(X, Z) \geq
    \frac{M_{\mathcal{F}}M_{\mathcal{G}}}{M_{\mathcal{X}}M_{\mathcal{Z}}} \sup_{\theta}\operatorname{Var}(h_{\theta}(X)).
\end{equation}
\end{theorem}
\com{
The proof of Theorem~\ref{thm:main_theorem} is in Appendix~\ref{sec:supp-proof-thm1} in the supplement. We use a result by Greenfeld and Shalit \citep{greenfeld2020robust} that links $\operatorname{HSIC}(X,Z)$ to the supremum of the covariance of bounded continuous functionals  over $\mathcal{X}$ and $\mathcal{Z}$.} 
Theorem~\ref{thm:main_theorem} indicates that  the regularizer $\HSIC(X,Z)$  at any intermediate layer  naturally suppresses the variability of the  output, i.e., the classifier prediction $h_\theta(X)$.
To see this, observe that by Chebyshev's inequality \citep{papoulis1989probability} the distribution of $h_\theta(X)$ concentrates around its mean when $\operatorname{Var}(h_{\theta}(X))$ approaches $0$.  As a result, bounding $\HSIC(X,Z)$ inherently also bounds the (global) variability of the classifier (across all parameters $\theta$). This observation motivates us to 
also maximize $\operatorname{HSIC} (Y, Z)$ to recover essential information useful for classification: 
if we want to achieve good adversarial robustness as well as good predictive accuracy, we have to strike a balance between $\operatorname{HSIC} (X, Z)$ and $\operatorname{HSIC} (Y, Z)$. This  perfectly aligns with the intuition behind the information bottleneck \citep{tishby2015deep} and the well-known accuracy-robustness trade off \citep{madry2017towards, zhang2019theoretically, tsipras2018robustness, raghunathan2020understanding}. 
We also confirm this experimentally: we observe that both  additional terms (the standard loss and $\HSIC(Y,Z)$) are necessary for ensuring good prediction performance in practice (see Table~\ref{tab:ablation}).

Most importantly, by further assuming that features are normal, we can show that HSIC  bounds the power of an arbitrary adversary, as defined in Eq.~\eqref{eq:adv_robustness}:
\begin{theorem} \label{thm:new_theorem}
\eqcom
Assume that $X \sim \mathcal{N}(0, \sigma^2 \mathbf{I})$. Then, under Assumptions~\ref{asm:cont} and~\ref{asm:universal},   we have:\footnote{Recall that for  functions $f,g:\reals\to\reals$ we have $f=o(g)$ if $\lim_{r\to 0}\frac{f(r)}{g(r)}=0$.} 
\begin{equation}
     \frac{ r \sqrt{-2 \log o(1)} \dx M_{\mathcal{Z}}}{ \sigma M_{\mathcal{F}}M_{\mathcal{G}}}\operatorname{HSIC}(X, Z) + o(r) \geq \E [|h_\theta(X+\delta) - h_\theta(X)|], \quad \text{for all}~\delta\in \mathcal{S}_r.
\end{equation}
\end{theorem}
The proof of Theorem~\ref{thm:new_theorem} can also be found in Appendix~\ref{sec:supp-proof-thm2} in the supplement. We again use the result by Greenfeld and Shalit \citep{greenfeld2020robust} along with Stein's Lemma~\citep{liu1994siegel}, that relates covariances of Gaussian r.v.s and their functions to expected gradients. In particular, we apply Stein's Lemma to the bounded functionals considered by Greenfeld and Shalit by using a truncation argument. 
Theorem~\ref{thm:new_theorem} implies that $\HSIC(X, Z)$ indeed bounds the output perturbation produced by an arbitrary adversary: suppressing HSIC sufficiently can ensure that the adversary cannot alter the output significantly, in expectation.  In particular, if 
$    \operatorname{HSIC}(X, Z) = o\left(\frac{\sigma M_{\mathcal{F}}M_{\mathcal{G}}}{ \sqrt{-2 \fcom{\log o(1)}} \dx M_{\mathcal{Z}}}\right),$
then 
 $   \lim_{r\to 0} \sup_{\delta \in  \mathcal{S}_r} {\E [|h_\theta(X+\delta) - h_\theta(X)|]}/{r} =0,$
i.e., the output is almost constant under small input perturbations. 

\section{Experiments}\label{sec:experiments}
 
\subsection{Experimental Setting}\label{sec:experient-setup}
We experiment with three standard datasets, MNIST \citep{mnist}, CIFAR-10 \citep{cifar} and CIFAR-100 \citep{cifar}.
~We use a 4-layer LeNet \citep{madry2017towards} for MNIST, ResNet-18 \citep{he2016deep} and WideResNet-28-10 \citep{wideresnet} for CIFAR-10, and WideResNet-28-10 \citep{wideresnet} for CIFAR-100. We use cross-entropy as loss $\loss(\theta)$. \com{Licensing information for all existing assets can be found in Appendix~\ref{sec:licensing} in the supplement.}

\noindent\textbf{Algorithms.}
We compare \emph{\hb} to the following non-adversarial learning algorithms: \emph{Cross-Entropy (CE)}, \emph{Stage-Wise HSIC Bottleneck (SWHB)} \citep{ma2020hsic}, \emph{XIC} \citep{greenfeld2020robust}, and \emph{Variational Information Bottleneck (VIB)} \citep{alemi2016deep}. We also incorporate \hb to several adversarial learning algorithms, as described in Section~\ref{sec:combine-hb}, and compare against the original methods, without the \hb penalty. The adversarial methods we use are: \emph{Projected Gradient Descent (PGD)} \citep{madry2017towards}, \emph{TRADES} \citep{zhang2019theoretically}, and \emph{MART} \citep{wang2019improving}. Further details and  parameters can be found in Appendix~\ref{sec:supp-algorithms} in the supplement. 




\noindent\textbf{Performance Metrics.} 
For all methods, we evaluate the obtained model $h_\theta$ via the following metrics: (a) \emph{Natural} (i.e., clean test data) accuracy, and adversarial robustness via test accuracy under (b) \emph{FGSM}, the fast gradient sign attack \citep{goodfellow2014explaining}, (c) \emph{PGD}$^m$, the PGD attack with $m$ steps used for the internal PGD optimization \citep{madry2017towards}, (d) \emph{CW}, the CW-loss within the PGD framework \citep{cw}, and (e) \emph{AA}, AutoAttack \citep{autoattack}. All five metrics are reported in percent (\%) accuracy. Following prior   literature, we set step size to 0.01 and radius $r=0.3$ for MNIST, and step size as $2/255$ and $r=8/255$ for CIFAR-10 and CIFAR-100. All attacks happen during the test phase and have full access to model parameters (i.e., are white-box attacks). All experiments are carried out on a Tesla V100 GPU with 32 GB memory and 5120 cores.

\subsection{Results}

\begin{table*}[!t]
    \centering
    \setlength{\extrarowheight}{.2em}
    \setlength{\tabcolsep}{3pt}
    \small
    \caption{Natural test accuracy (in \%), adversarial robustness ((in \%) on FGSM, PGD, CW, and AA attacked test examples) on MNIST and CIFAR-100 of \textbf{[row i, iii, v] adversarial learning baselines} and \textbf{[row ii, iv, vi] combining \hb with each correspondingly}. Each result is the average of five runs.}
    \label{tab:adv-competing}
    \resizebox{1\textwidth}{!}{
    \begin{tabular}{||c || c c c c c c||c c c c c c||}
        \hline
        \multirow{2}{*}{Methods} & \multicolumn{6}{c||}{MNIST by LeNet} & \multicolumn{6}{c||}{CIFAR-100 by WideResNet-28-10} \\
        \cline{2-13}
        & Natural & FGSM & PGD$^{20}$ & PGD$^{40}$ & CW & AA & Natural & FGSM & PGD$^{10}$ & PGD$^{20}$ & CW & AA \\
        \hline
        \hline
         PGD            & 98.40 & 93.44 & 94.56 & 89.63 & 91.20 & 86.62 & 59.91 & 29.85 & 26.05 & 25.38 & 22.28 & 20.91 \\
        \hb + PGD       & \textbf{98.66} & \textbf{96.02} & \textbf{96.44} & \textbf{94.35} & \textbf{95.10} & \textbf{91.57}
                        & \textbf{63.84} & \textbf{31.59} & \textbf{27.90} & \textbf{27.21} & \textbf{23.23} & \textbf{21.61} \\
        \hline
        \hline
        TRADES          & \textbf{97.64} & 94.73 & 95.05 & 93.27 & 93.05 & 89.66 & 60.29 & 34.19 & 31.32 & 30.96 & 28.20 & 26.91 \\
        \hb + TRADES    & \textbf{97.64} & \textbf{95.23} & \textbf{95.17} & \textbf{93.49} & \textbf{93.47}& \textbf{89.99} 
                        & \textbf{60.55} & \textbf{34.57} & \textbf{31.96} & \textbf{31.57} & \textbf{28.72}& \textbf{27.46
                        } \\
        \hline
        \hline
        MART            & \textbf{98.29} & 95.57 & 95.23 & 93.55 & 93.45 & 88.36 & 58.42 & 32.94 & 29.17 & 28.19 & 27.31 & 25.09 \\
        \hb + MART      & 98.23 & \textbf{96.09} & \textbf{96.08} & \textbf{94.64} & \textbf{94.62} & \textbf{89.99} 
                        & \textbf{58.93} & \textbf{33.49} & \textbf{30.72} & \textbf{30.16} & \textbf{28.89}& \textbf{25.21} \\
        \hline
    \end{tabular}}
\end{table*}

\begin{table*}[!t]
    \centering
    \setlength{\extrarowheight}{.2em}
    \setlength{\tabcolsep}{3pt}
    \small
    \caption{Natural test accuracy (in \%), adversarial robustness ((in \%) on FGSM, PGD, CW, and AA attacked test examples) on CIFAR-10 by ResNet-18 and WideResNet-28-10 of \textbf{[row i, iii, v] adversarial learning baselines} and \textbf{[row ii, iv, vi] combining \hb with each correspondingly}. Each result is the average of five runs.}
    \label{tab:adv-competing-cifar10}
    \resizebox{1\textwidth}{!}{
    \begin{tabular}{||c || c c c c c c ||c c c c c c||}
        \hline
        \multirow{2}{*}{Methods} & \multicolumn{6}{c||}{CIFAR-10 by ResNet-18} & \multicolumn{6}{c||}{CIFAR-10 by WideResNet-28-10} \\
        \cline{2-13}
        & Natural & FGSM & PGD$^{10}$ & PGD$^{20}$ & CW & AA & Natural & FGSM & PGD$^{10}$ & PGD$^{20}$ & CW & AA \\
        \hline
        \hline
         PGD            & 84.71 & 55.95 & 49.37 & 47.54 & 41.17 & 43.42 & 86.63 & 58.53 & 52.21 & 50.59 & 49.32 & 47.25 \\
        \hb + PGD       & \textbf{85.73} & \textbf{57.13} & \textbf{49.63} & \textbf{48.32} & \textbf{41.80} & \textbf{44.46} & \textbf{87.91} & \textbf{59.69} & \textbf{52.72} & \textbf{51.17} & \textbf{49.52} & \textbf{47.60} \\
        \hline
        \hline
        TRADES          & 84.07 & 58.63 & 53.21 & 52.36 & 50.07 & 49.38 & \textbf{85.66} & 61.55 & 56.62 & 55.67 & 54.02 & 52.71 \\
        \hb + TRADES    & \textbf{84.10} & \textbf{58.97} & \textbf{53.76} & \textbf{52.92} & \textbf{51.00} & \textbf{49.43} & 85.61 & \textbf{62.20} & \textbf{57.30} & \textbf{56.51} & \textbf{54.89} & \textbf{53.53} \\
        \hline
        \hline
        MART            & 82.15 & 59.85 & 54.75 & 53.67 & 50.12 & 47.97 & \textbf{85.94} & 59.39 & 51.30 & 49.46 & 47.94 & 45.48 \\
        \hb + MART      & \textbf{82.44} & \textbf{59.86} & \textbf{54.84} & \textbf{53.89} & \textbf{50.53} & \textbf{48.21} & 85.52 & \textbf{60.54} & \textbf{53.42} & \textbf{51.81} & \textbf{49.32} & \textbf{46.99} \\
        \hline
    \end{tabular}}
\end{table*}

\noindent\textbf{Combining \hb with Adversarial Examples.}
We show how \hb can be used to improve robustness when used as a regularizer, as described in \Cref{sec:combine-hb}, along with state-of-the-art adversarial learning methods. We run each experiment by five times and report the mean natural test accuracy and adversarial robustness of all models on MNIST, CIFAR-10, and CIFAR-100 datasets by four architectures in Table \ref{tab:adv-competing} and Table \ref{tab:adv-competing-cifar10}. 
Combined with all adversarial training baselines, \hb \emph{consistently improves adversarial robustness against all types of attacks on all datasets}. \com{The resulting improvements are larger than 2 standard deviations (that range between 0.05-0.2) in most cases}; we report the \com{results with standard deviations} in Appendix~\ref{sec:supp-errorbar} in the supplement.  Although natural accuracy is generally restricted by the trade-off between robustness and accuracy \citep{zhang2019theoretically}, we observe that incorporating \hb comes with an actual improvement over natural accuracy in most cases.

\begin{figure*}[!t]
\begin{tabular}{c c c}
   \centering
   \small
   \includegraphics[width=0.3\columnwidth]{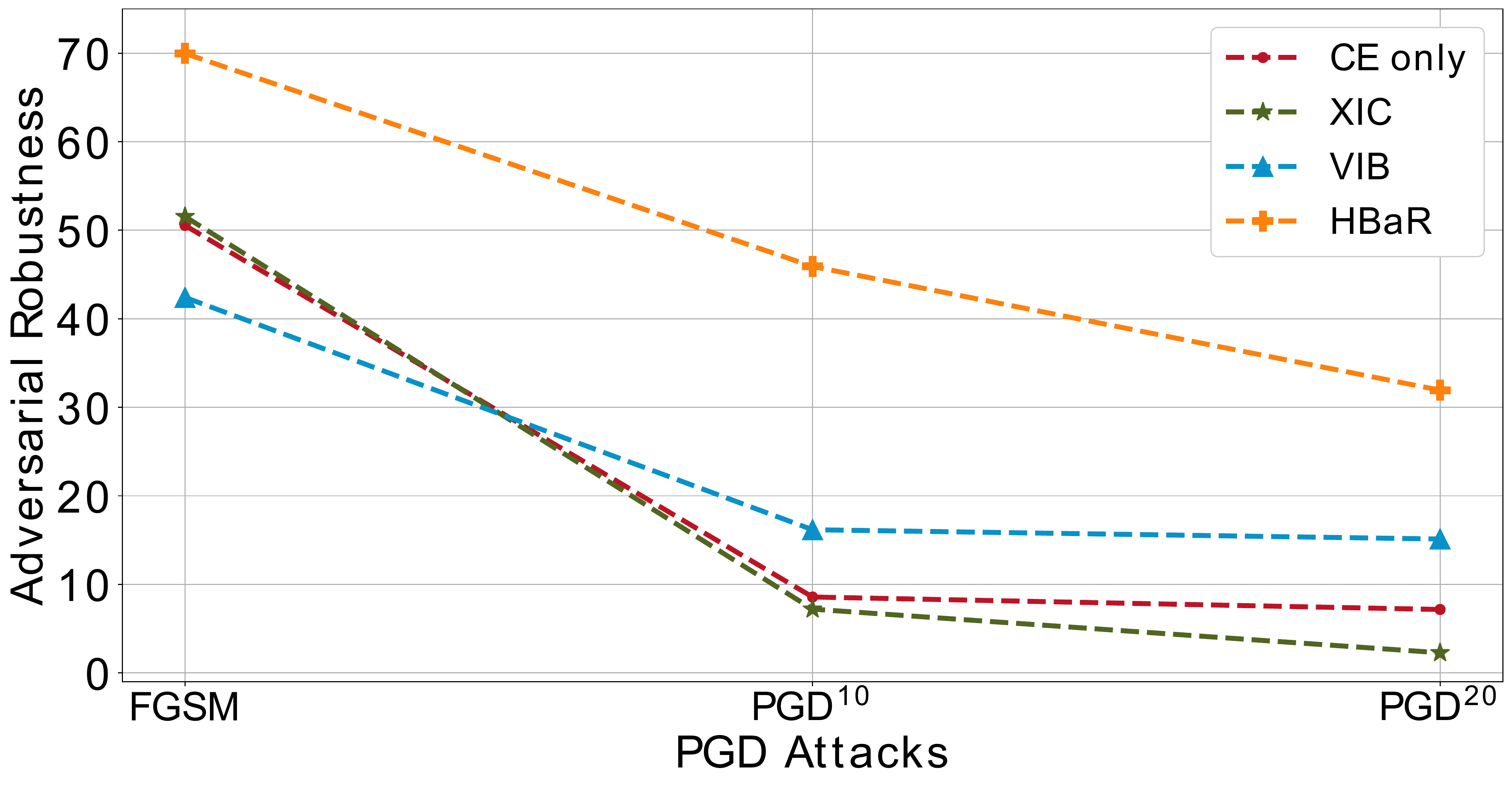} &
   \includegraphics[width=0.3\columnwidth]{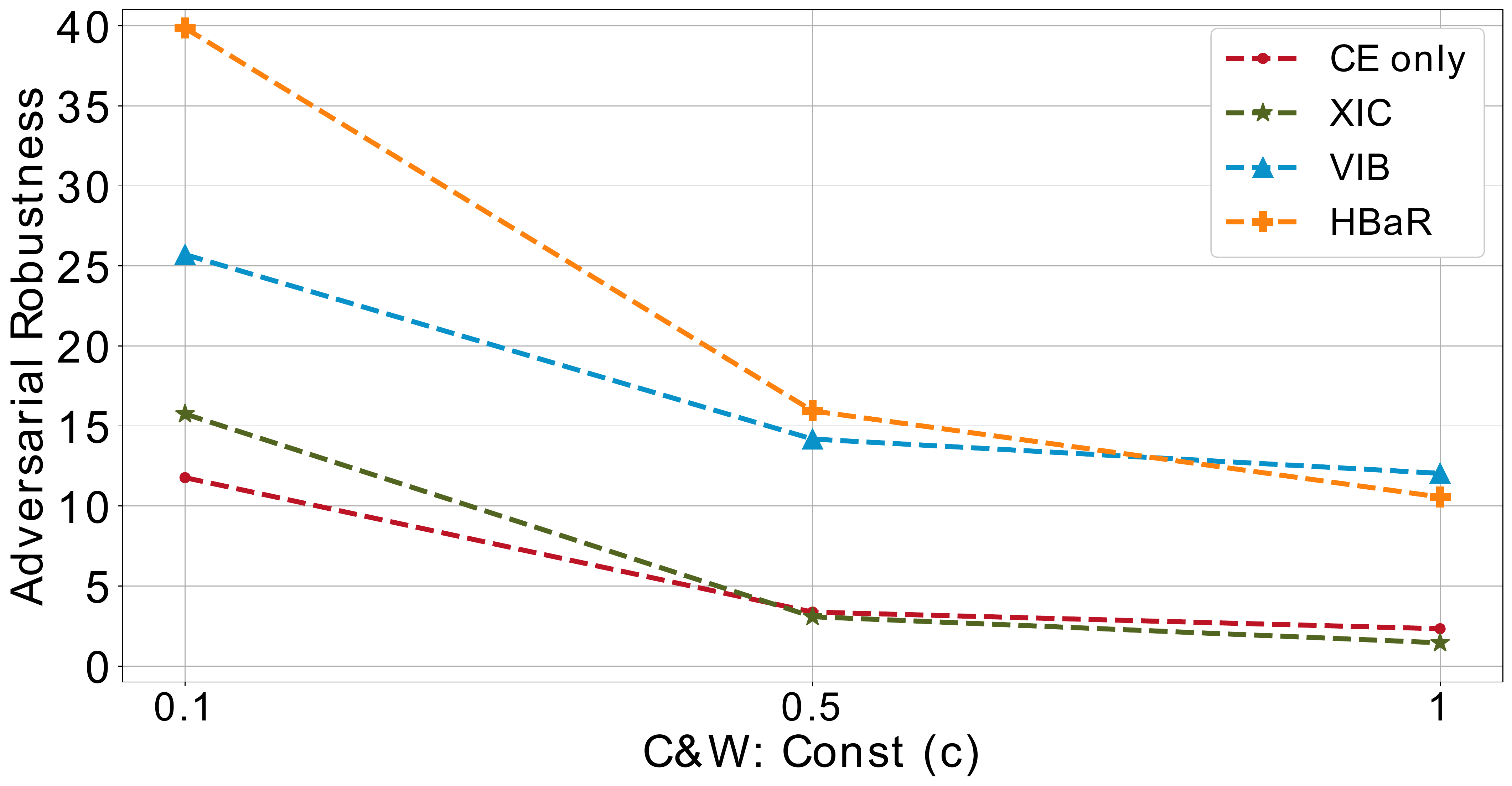} &
   \includegraphics[width=0.3\columnwidth]{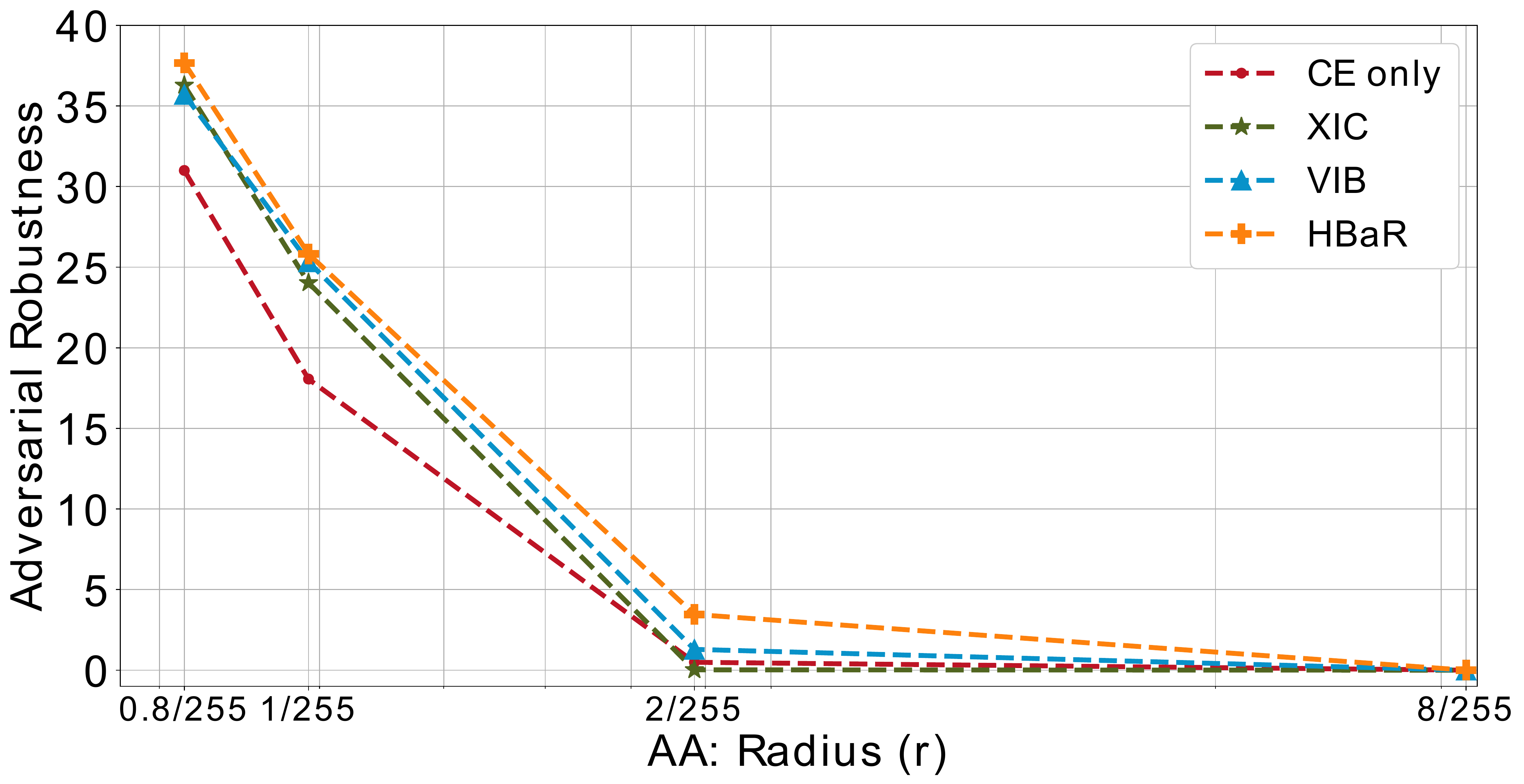} \\
   (a) PGD attacks & (b) CW attack & (c) AA using different radius \\
\end{tabular}
	\caption{CIFAR-10 by ResNet-18: Adversarial robustness of \textbf{IB-based baselines} and \textbf{proposed \hb} under (a) PGD attacks, (b) CW attack by various of constant $c$, and (c) AA using different radius. Interestingly\com{, while achieving the highest adversarial robustness under almost all cases}, \hb achieves natural accuracy (95.27\%) comparable to CE (95.32\%) \fcom{which is much higher than VIB (92.35\%), XIC (92.93\%) and SWHB (59.18\%).}}
	\label{fig:hsic-competing}
\end{figure*}

\noindent\textbf{Adversarial Robustness Analysis without Adversarial Training.}\label{sec:experiments-hsiccompare}
Next, we show that \hb can achieve modest robustness even without adversarial examples during training. We evaluate the robustness of \hb on CIFAR-10 by ResNet-18 against various adversarial attacks, and compare \hb with other information bottleneck penalties without adversarial training in Figure \ref{fig:hsic-competing}. Specifically, we compare the robustness of \hb with other IB-based methods under various attacks and hyperparameters. Our proposed \hb achieves the best \com{overall} robustness against all three types of attacks while attaining competitive natural test accuracy. Interestingly, \hb achieves natural accuracy (95.27\%) comparable to CE (95.32\%) which is much higher than VIB (92.35\%), XIC (92.93\%) and SWHB (59.18\%). \fcom{We observe SWHB underperforms \hb on CIFAR-10 for both natural accuracy and robustness.} One possible explanation may be that when the model is deep, minimizing HSIC without backpropagation, as in SWHB, does not suffice to transmit the learned information across layers. Compared to SWHB, HBaR backpropagates over the HSIC objective through each intermediate layer and computes gradients only once in each batch, \fcom{improving accuracy and robustness while reducing computational cost significantly. }

\begin{figure*}[!t]
   \centering
   \includegraphics[width=0.9\columnwidth]{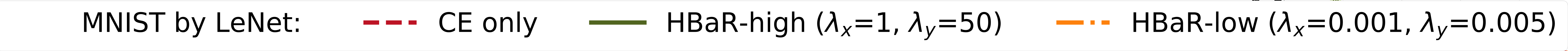} \\
   \vspace{-2pt}
   \begin{tabular}{c c c c}
    	\includegraphics[width=0.22\columnwidth]{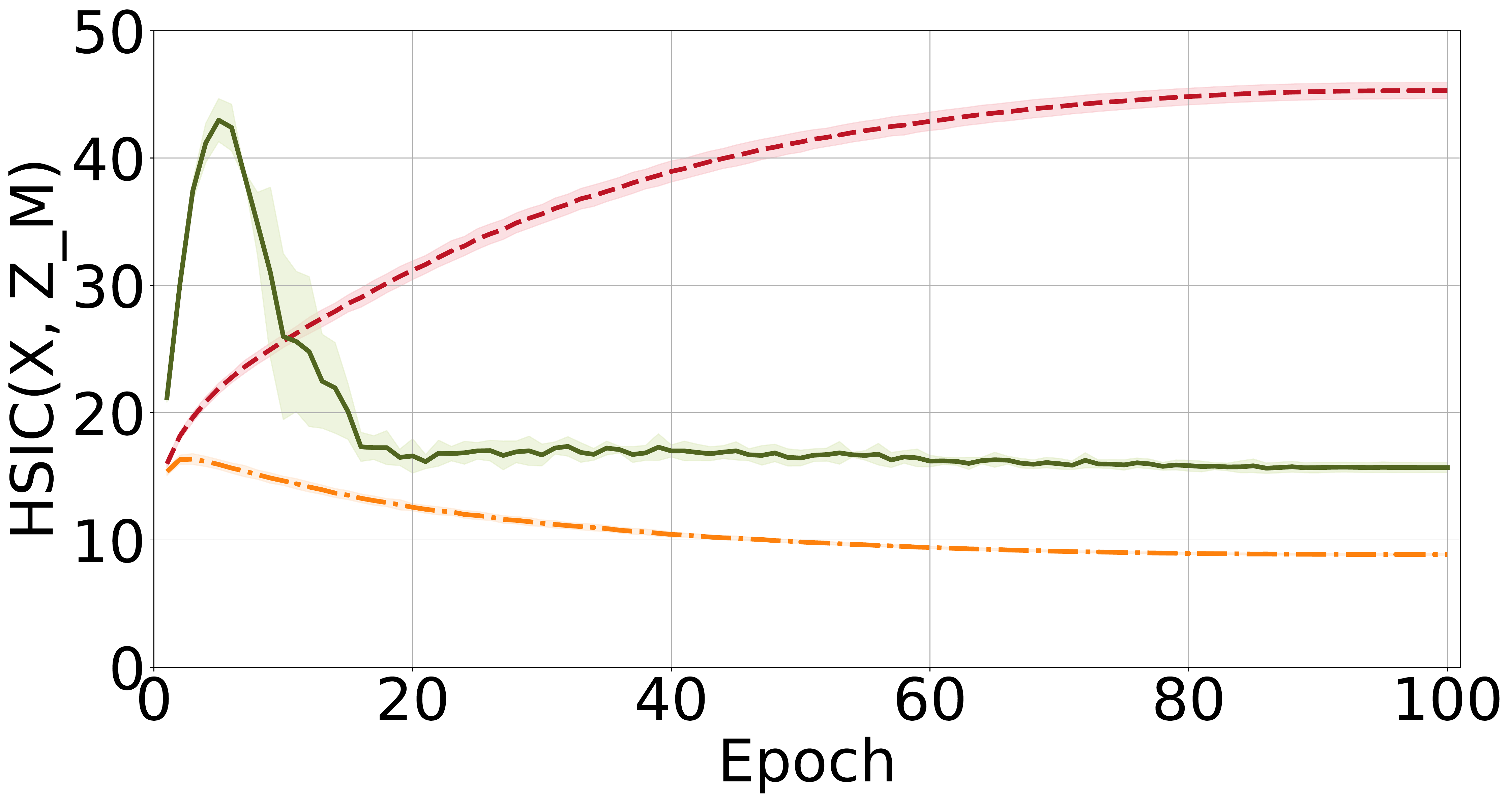} &
    	\includegraphics[width=0.22\columnwidth]{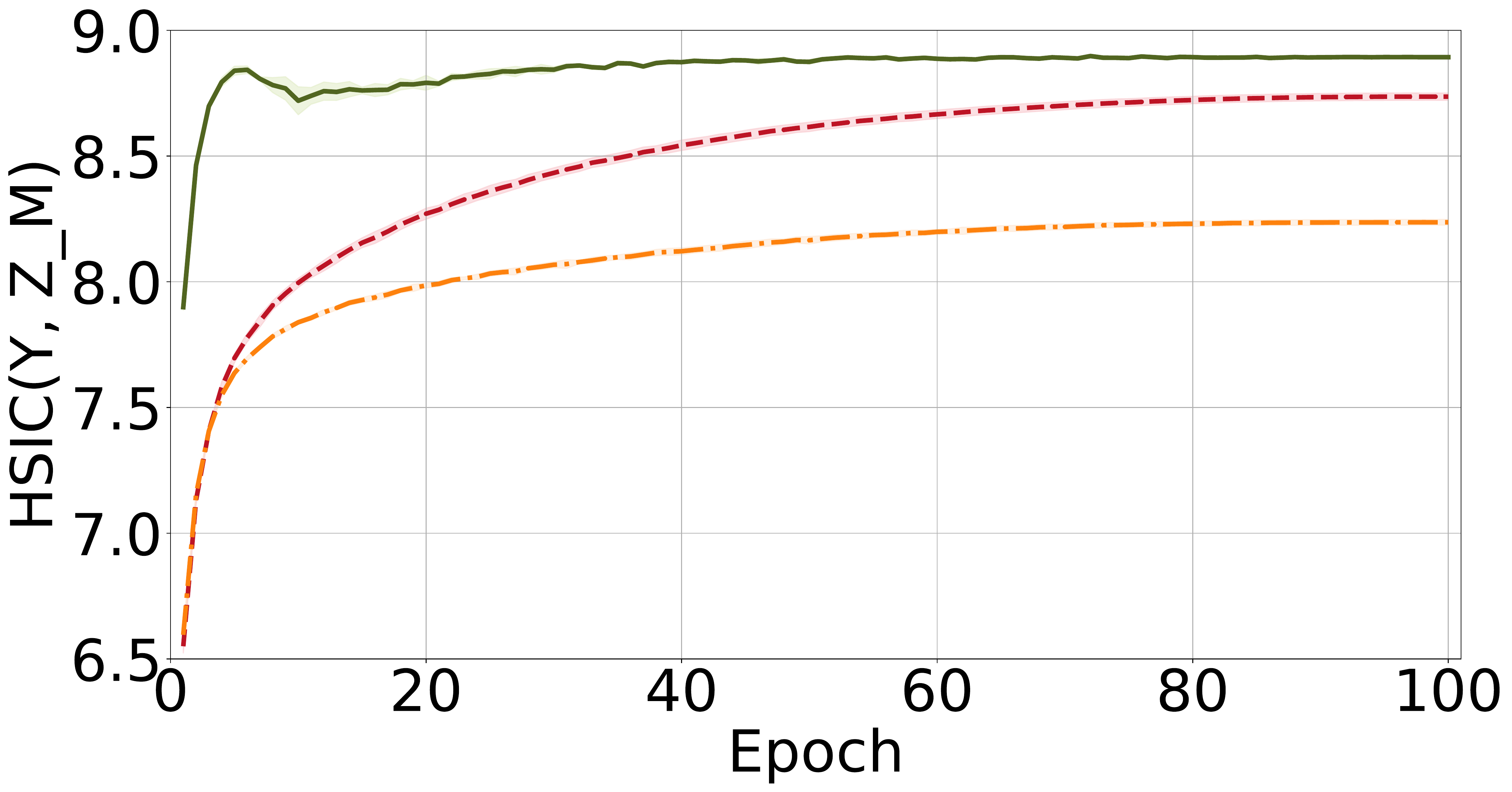} &
    	\includegraphics[width=0.22\columnwidth]{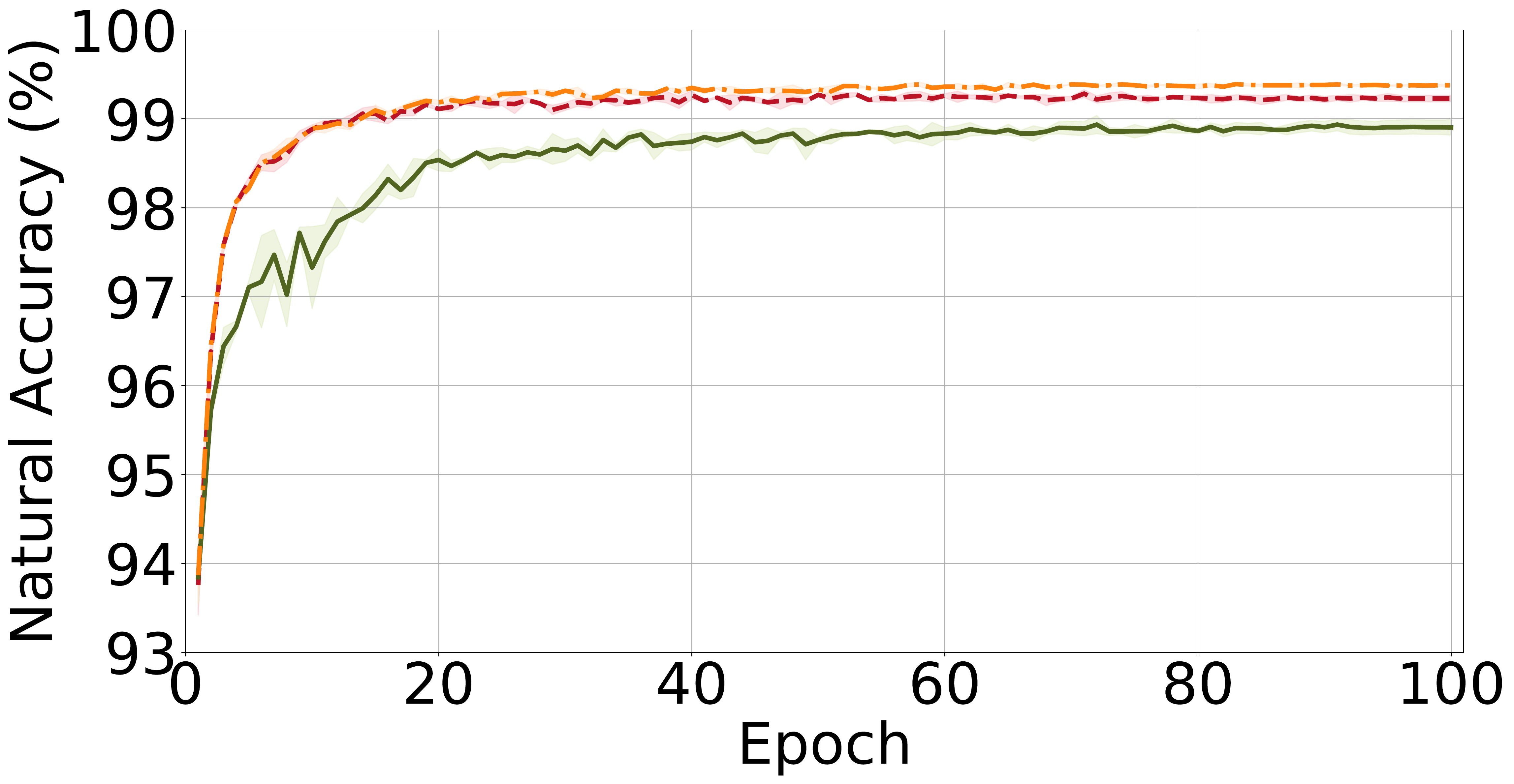} &
    	\includegraphics[width=0.22\columnwidth]{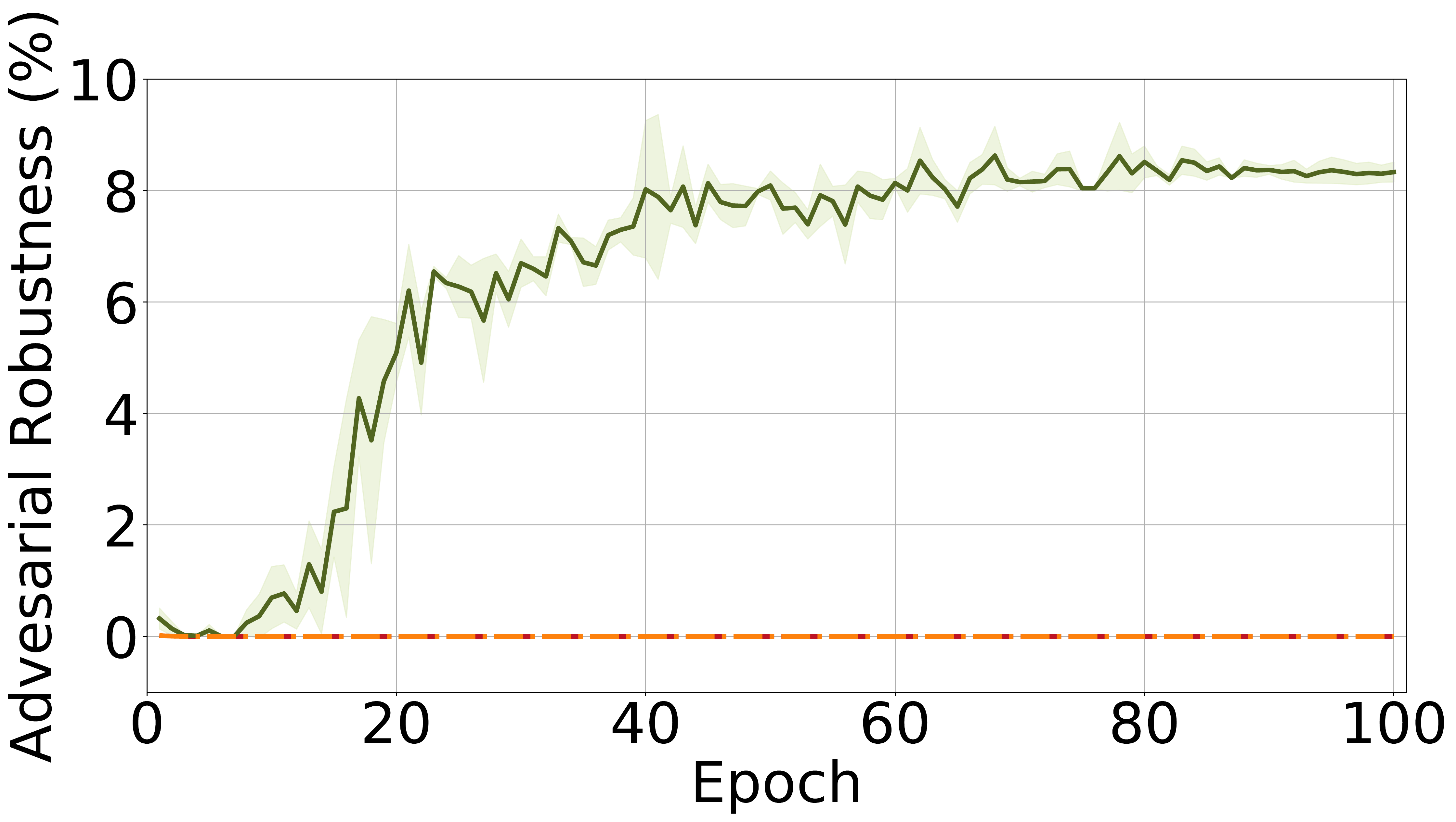} \\
	\end{tabular}
	\vspace{4pt}
	
	\includegraphics[width=\columnwidth]{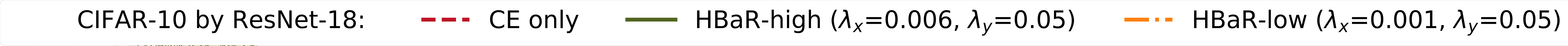} \\
	\vspace{-2pt}
    \begin{tabular}{c c c c}
        \includegraphics[width=0.22\columnwidth]{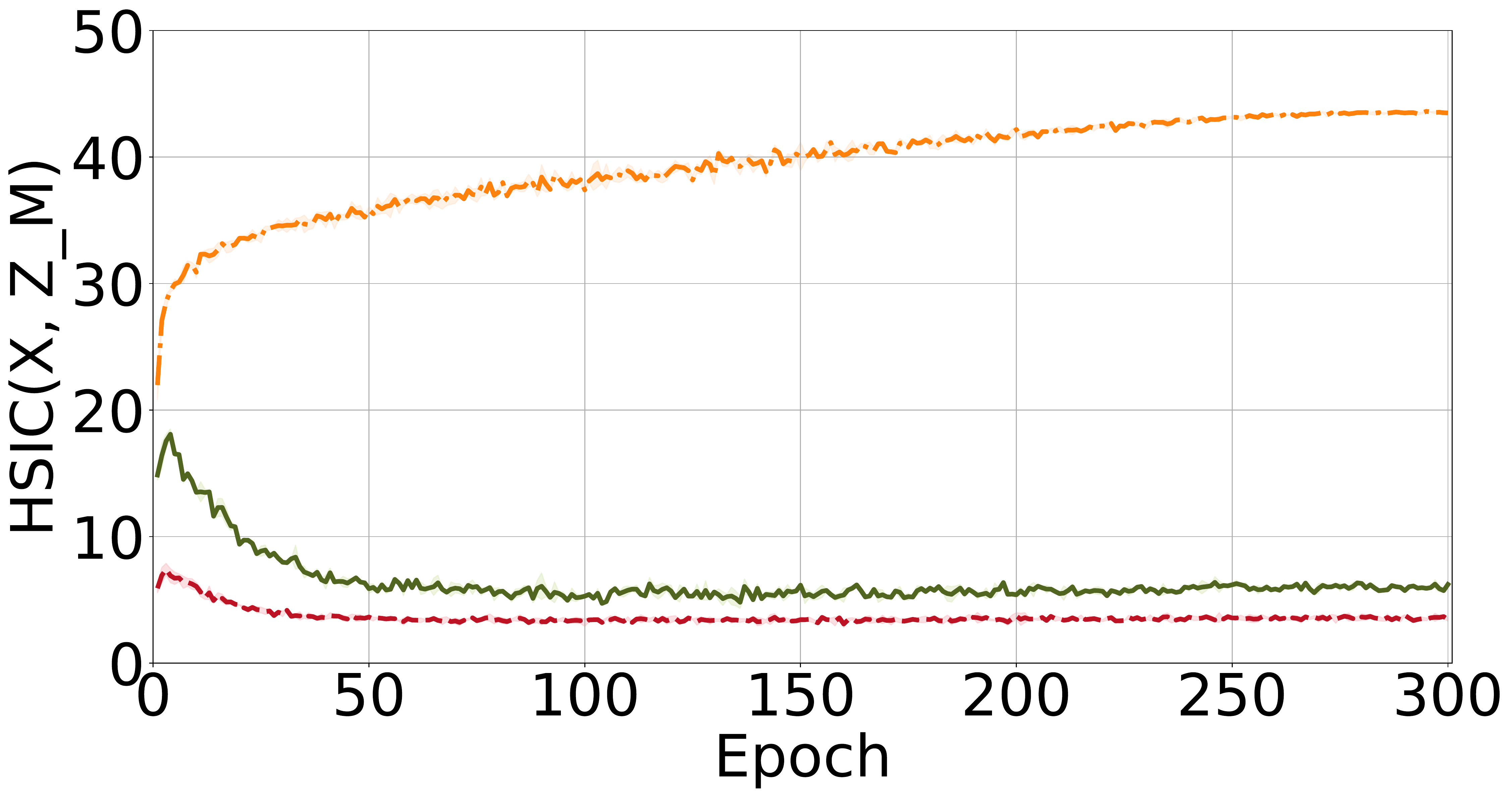} &
    	\includegraphics[width=0.22\columnwidth]{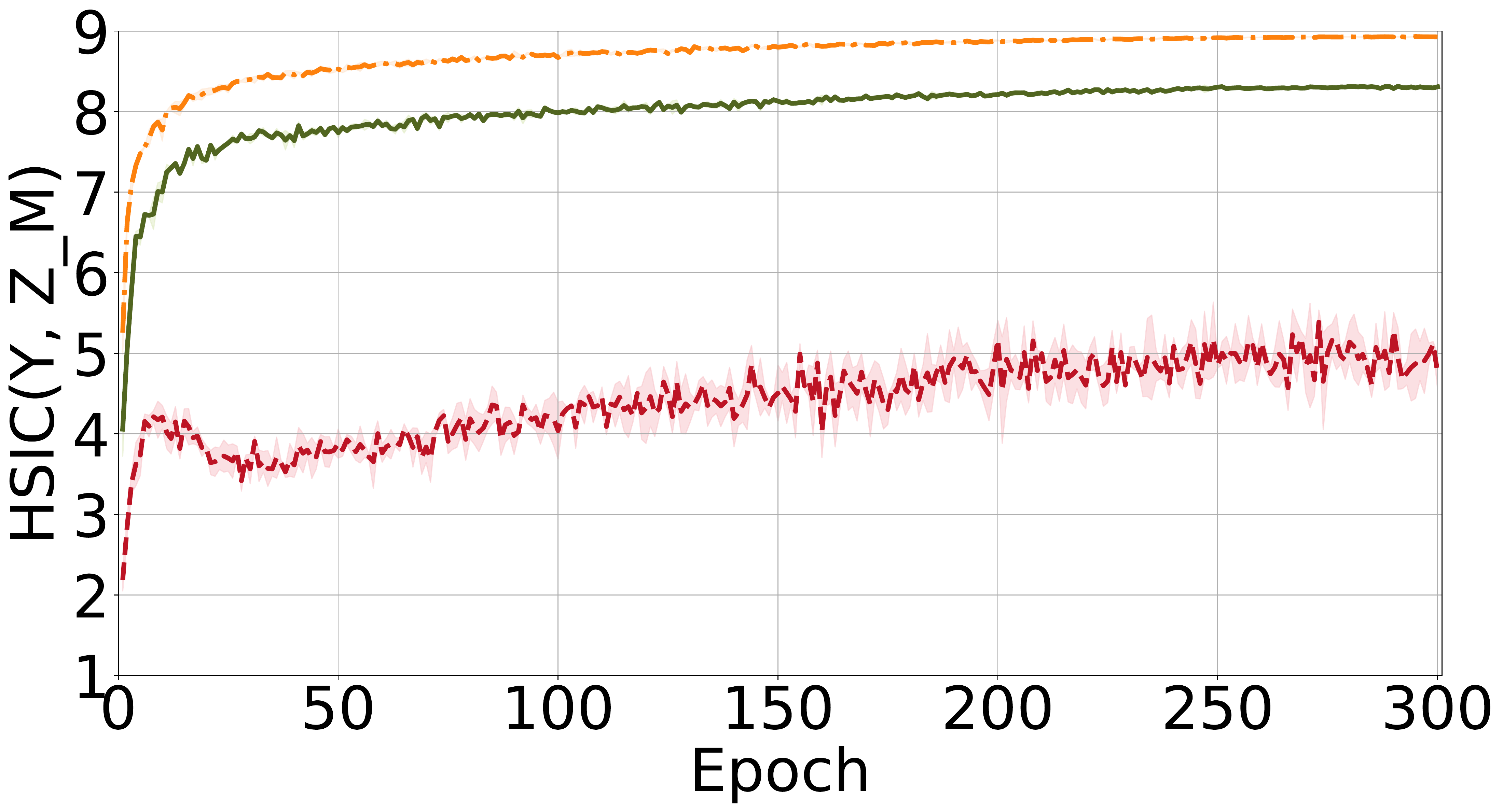} &
    	\includegraphics[width=0.22\columnwidth]{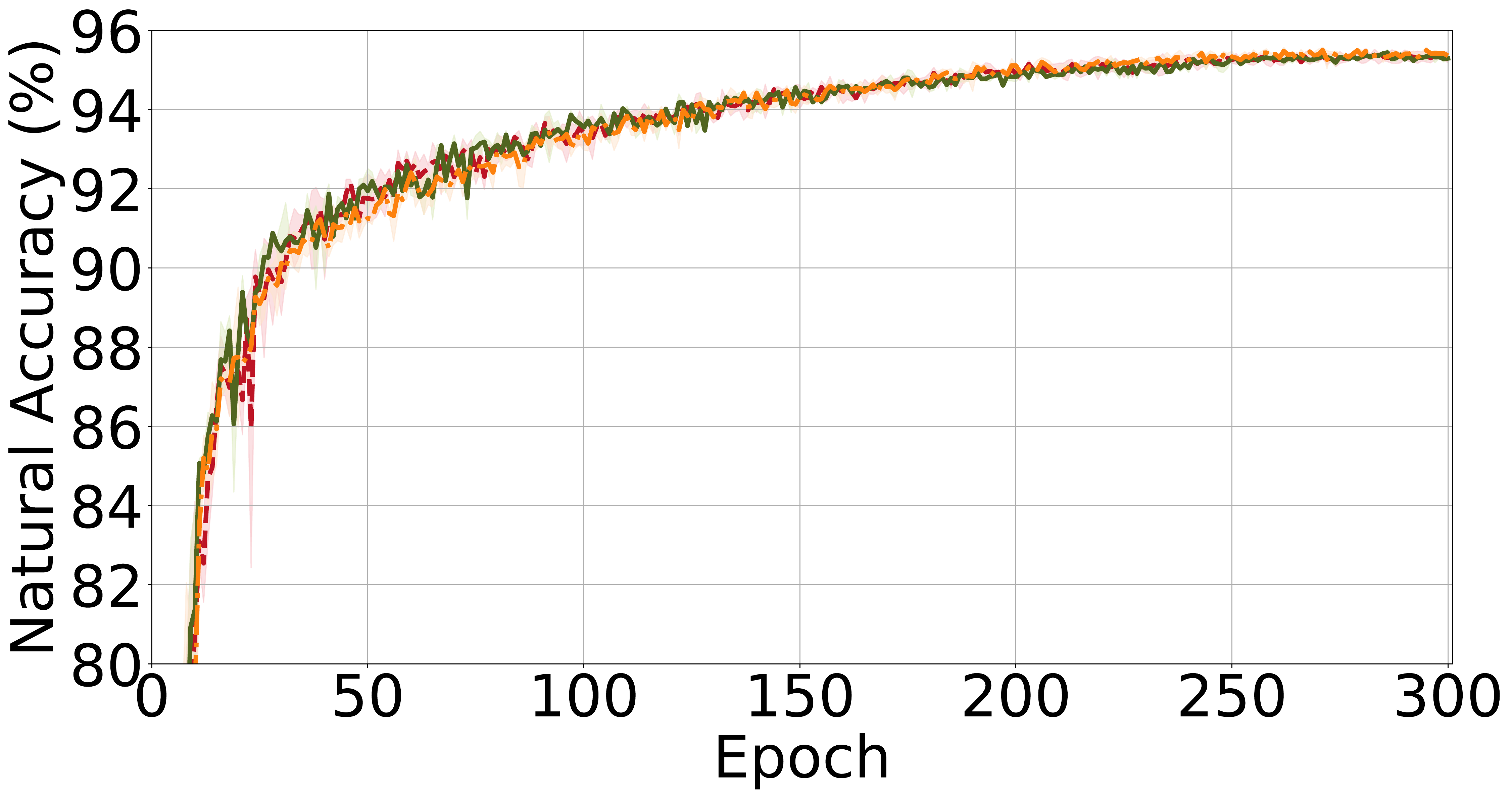} &
    	\includegraphics[width=0.22\columnwidth]{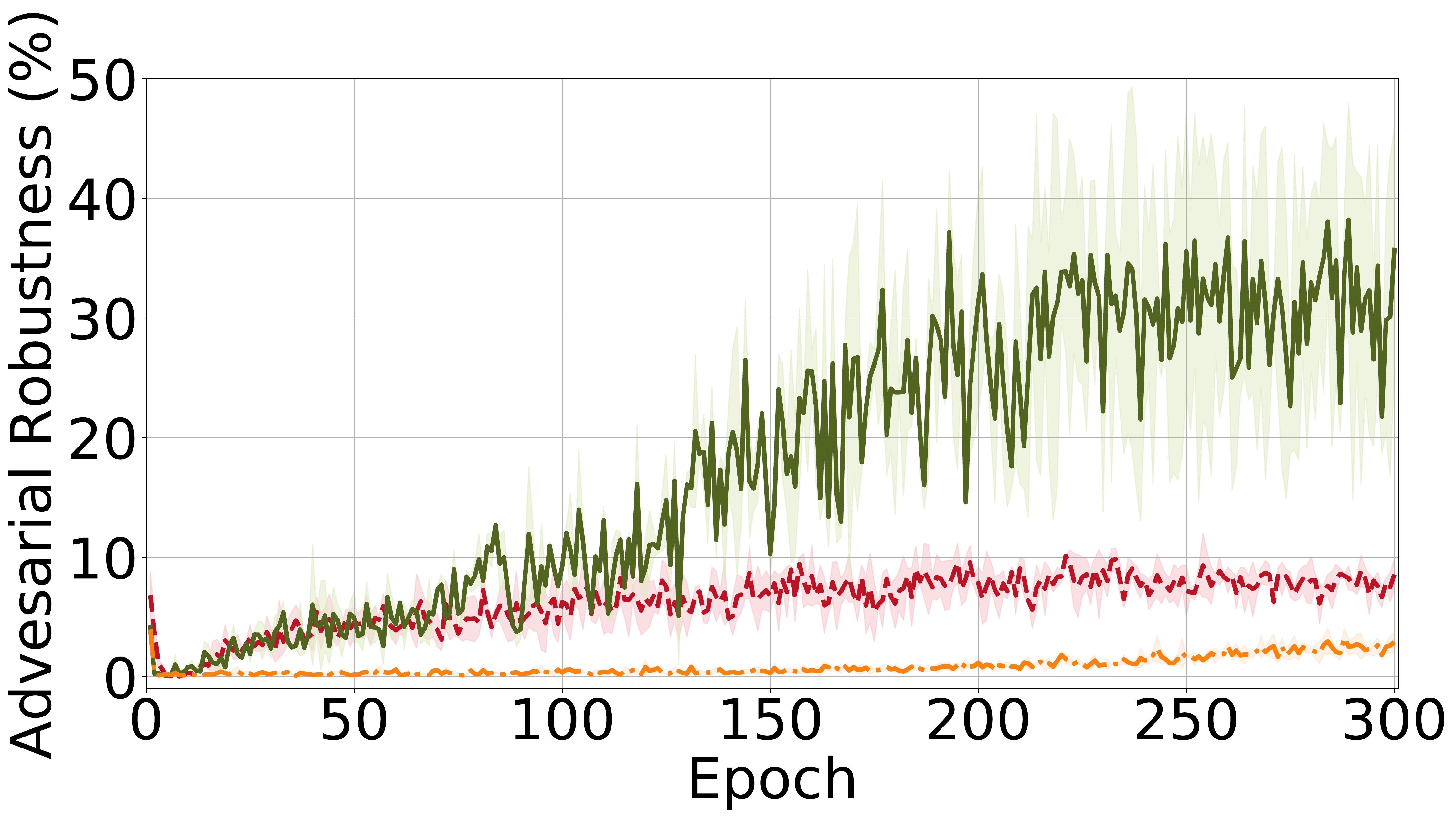} \\
    	(a) HSIC$(X,Z_M)$ & (b) HSIC$(Y,Z_M)$ & (c) Natural Accuracy & (d) Adv. Robustness \\
	\end{tabular}
	\caption{Visualization of the \hb quantities (a) HSIC$(X,Z_M)$, (b) HSIC$(X,Z_M)$, (c) natural test accuracy, and (d) adversarial robustness against PGD attack (PGD$^{40}$ and PGD$^{20}$ on MNIST and CIFAR-10, respectively) as a function of training epochs, on  MNIST  by LeNet (top) and CIFAR-10 by ResNet (bottom). Different colored lines correspond to CE, \hb-high (\hb with high weights $\lambda$), and \hb-low (\hb\ method small weighs $\lambda$). \hb-low parameters are selected so that the values of the loss $\loss$ and each of the $\HSIC$ terms are close after the first epoch.}
	\label{fig:metrics_versus_epochs}
\end{figure*}

\begin{figure*}[!t]
   \centering
   \small
   \includegraphics[width=\columnwidth]{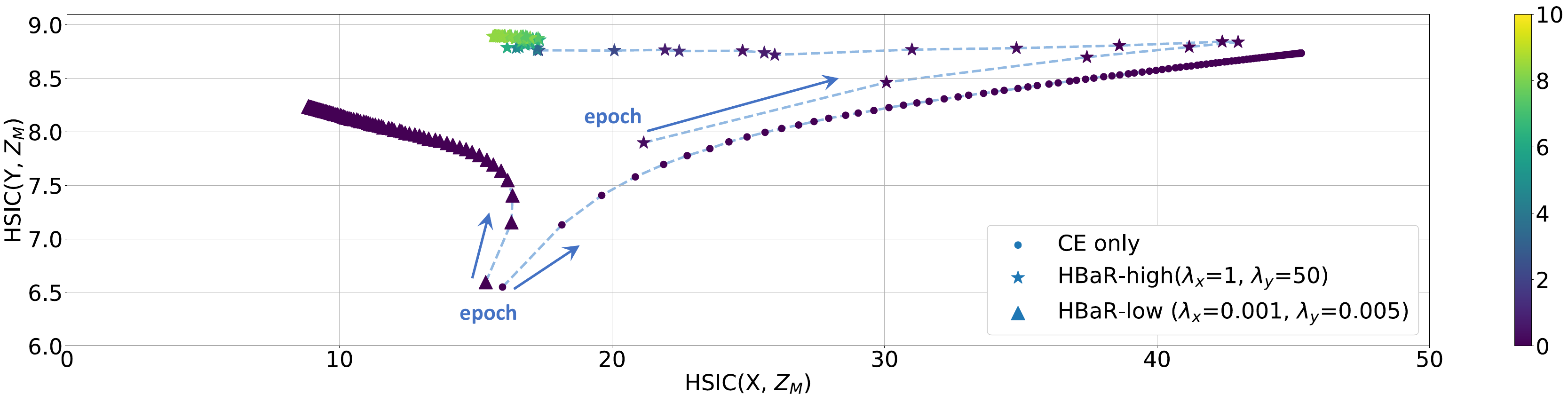} \\
   (a) MNIST by LeNet \\
   \includegraphics[width=\columnwidth]{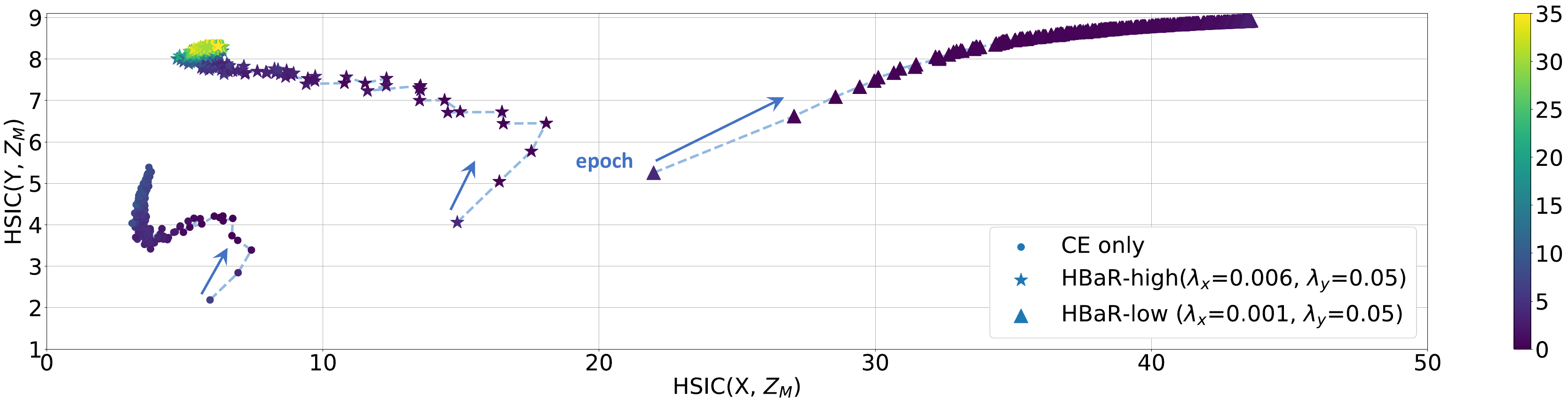} \\
	(b) CIFAR-10 by ResNet-18 \\
	\caption{HSIC plane dynamics versus adversarial robustness. The x-axis plots HSIC between the last intermediate layer $Z_M$ and the input $X$, while the y-axis plots HSIC between $Z_M$ and the output $Y$. The color scale indicates adversarial robustness against PGD attack (PGD$^{40}$ and PGD$^{20}$ on MNIST and CIFAR-10, respectively). The arrows indicate dynamic direction w.r.t. training epochs. Each marker in the figures represents a different setting: \textbf{dots}, \textbf{stars}, and \textbf{triangles} represent CE-only, \hb-high, and \hb-low, respectively, compatible with the definition in Figure \ref{fig:metrics_versus_epochs}.}
	\label{fig:hsic_plain}
\end{figure*}

\noindent\textbf{Synergy between HSIC Terms.} Focusing on $Z_\nol$,  the last latent layer, Figure~\ref{fig:metrics_versus_epochs} shows the evolution per epoch of: (a) $\operatorname{HSIC}(X,Z_\nol)$, (b) $\operatorname{HSIC}(Y,Z_\nol)$, (c) natural accuracy (in \%), and (d) adversarial robustness (in \%) under PGD attack on MNIST and CIFAR-10. Different lines correspond to CE, \hb-high (\hb with high weights $\lambda$), and \hb-low (\hb\ method small weighs $\lambda$). \hb-low parameters are selected so that the values of the loss $\loss$ and each of the $\HSIC$ terms are close after the first epoch.
Figure~\ref{fig:metrics_versus_epochs}(c) illustrates that all three settings achieve good natural accuracy on both datasets. However, in  Figure~\ref{fig:metrics_versus_epochs}(d), only \hb-high, that puts sufficient weight on $\HSIC$ terms, attains relatively high adversarial robustness. In  Figure~\ref{fig:metrics_versus_epochs}(a), we see that CE leads to high $\HSIC(X,Z_\nol)$ for the shallow LeNet, but low in the (much deeper) ResNet-18, even lower than \hb-low. Moreover, we also see that the best performer in terms of adversarial robustness, \hb-high, lies in between the other two w.r.t.~$\operatorname{HSIC}(X, Z_M)$. Both of these observations indicate the importance of the $\operatorname{HSIC}(Y, Z_M)$  penalty:  minimizing $\operatorname{HSIC}(X, Z_M)$ appropriately leads to good adversarial robustness, but coupling learning to labels via the third term is integral to maintaining useful label-related information in latent layers, thus resulting in good adversarial robustness. Figure~\ref{fig:metrics_versus_epochs}(b) confirms this, as \hb-high achieves relatively high $\operatorname{HSIC}(Y, Z_M)$ on both datasets. 
    
Figure~\ref{fig:hsic_plain} provides another perspective of the same experiments via the learning dynamics on the HSIC plane. We again observe that the best performer in terms of robustness \hb-high lies in between the other two methods, crucially attaining a much higher $\operatorname{HSIC}(Y, Z_M)$ than \hb-low. Moreover, for both \hb methods, we clearly observe the two distinct optimization phases first observed by \fcom{Shwartz-Ziv and Tishby}~\citep{shwartz2017opening} in the context of the mutual information bottleneck:  the \textit{fast empirical risk minimization phase}, where the neural network tries to learn a meaningful representation by increasing $\operatorname{HSIC}(Y, Z_\nol)$ regardless of information redundancy ($\operatorname{HSIC}(X, Z_M)$ increasing), and the \textit{representation compression phase}, where the neural network turns its focus onto compressing the latent representation by minimizing $\operatorname{HSIC}(X, Z_M)$, while  maintaining highly label-related information. Interestingly, the \hb penalty produces the two-phase behavior even though our networks use ReLU activation functions; Shwartz et al.~\citep{shwartz2017opening} only observed these two optimization phases on neural networks with tanh activation functions, a phenomenon further confirmed by Saxe et al.~\citep{saxe2019information}.  

\begin{table*}[!t]
    \centering
    \setlength{\extrarowheight}{.2em}
    \setlength{\tabcolsep}{0.9pt}
    \caption{Ablation study on \hb. Rows [i-iv] indicate the effect of removing each component of the learning objective defined in Eq.\eqref{eq:obj} (row [v]). We evaluate each objective over $\operatorname{HSIC}(X,Z_M)$, $\operatorname{HSIC}(Y,Z_M)$, natural test accuracy (in \%), and adversarial robustness (in \%) against PGD$^{40}$ and PGD$^{20}$ on MNIST and CIFAR-10 respectively. We set $\lambda_x$ as 1 and 0.006, $\lambda_y$ as 50 and 0.05, for MNIST and CIFAR-10 respectively.}
    \label{tab:ablation}
    \resizebox{1\textwidth}{!}{
    \begin{tabular}{||c | c || c c | c c || c c | c c ||}
        \hline
        \multirow{3}{*}{Rows} & \multirow{3}{*}{Objectives} & \multicolumn{4}{c||}{MNIST by LeNet} & \multicolumn{4}{c||}{CIFAR-10 by ResNet-18} \\
        \cline{3-10}
        & & \multicolumn{2}{c|}{HSIC} & \multirow{2}{*}{Natural} & \multirow{2}{*}{PGD$^{40}$} & \multicolumn{2}{c|}{HSIC} & \multirow{2}{*}{Natural} & \multirow{2}{*}{PGD$^{20}$} \\
        & &$(X,Z_M)$&$(Y,Z_M)$& & &$(X,Z_M)$&$(Y,Z_M)$& &\\
        \hline
        \hline
        [i] & $\loss(\theta)$ & 45.29 & 8.73 & 99.23 & 0.00 & 3.45 & 4.76 & 95.32 & 8.57 \\
        
        [ii] & $\lambda_{x} \sum_j\operatorname{HSIC}(X, Z_j) -\lambda_{y} \sum_j\operatorname{HSIC}(Y, Z_j)$ 
        & 16.45 & 8.65 & 30.08 & 9.47 & 44.37 & 8.72 & 19.30  & 8.58 \\
        
        [iii] & $\loss(\theta)+\lambda_{x} \sum_j\operatorname{HSIC}(X, Z_j)$
        & 0.00  & 0.00 & 11.38 & 10.00 & 0.00 & 0.00 & 10.03 & 10.10 \\
        
        [iv] & $\loss(\theta) -\lambda_{y} \sum_j\operatorname{HSIC}(Y, Z_j)$
        & 56.38 & 9.00 & 99.33 & 0.00 & 43.71 & 8.93 & 95.50 & 1.90 \\
        \hline
        
        [v] & $\loss(\theta) +\lambda_{x} \sum_j\operatorname{HSIC}(X, Z_j) -\lambda_{y} \sum_j\operatorname{HSIC}(Y, Z_j)$
        & 15.68 & 8.89 & 98.90 & 8.33 & 6.07  & 8.30 & 95.35 & 34.85\\
        \hline
    \end{tabular}
    }
\end{table*}

\noindent\textbf{Ablation Study.} 
Motivated by the above observations, we turn our attention to how the three
terms in the loss function in Eq.~\eqref{eq:obj} affect \hb. 
As illustrated in Table~\ref{tab:ablation}, removing any part leads to either a significant natural accuracy or robustness degradation. Specifically, using $\loss(\theta)$ only (row [i]) lacks adversarial robustness; removing $\loss(\theta)$ (row [ii]) or the penalty on $Y$ (row [iii]) degrades natural accuracy significantly \fcom{(a similar result was also observed in \citep{amjad2018not})}; finally, 
removing the penalty on $X$ 
improves the natural accuracy while degrading adversarial robustness. The three terms combined together by proper hyperparameters $\lambda_x$ and $\lambda_y$ (row [v]) achieve both high natural accuracy and adversarial robustness. We provide a comprehensive ablation study on the sensitivity of $\lambda_x$ and $\lambda_y$ and draw conclusions in Appendix~\ref{sec:supp-ablation} in the supplement (Tables~\ref{tab:mnist-weight} and \ref{tab:cifar-weight}).\\

\section{Conclusions} \label{sec:conclusion}
We investigate the HSIC bottleneck as regularizer (HBaR) as a means to enhance adversarial robustness. We theoretically prove that HBaR suppresses the sensitivity of the classifier to adversarial examples while retaining its discriminative nature. 
One limitation of our method is that the robustness gain is modest when training with only natural examples. Moreover, a possible negative societal impact is overconfidence in adversarial robustness: over-confidence in the \emph{adversarially-robust} models produced by \hb as well as other defense methods may lead to overlooking their potential failure on newly-invented attack methods; this should be taken into account in safety-critical applications like healthcare~\citep{adv_health} or security~\citep{adv_surveillance}. \com{We extend the discussion on the limitations and potential negative societal impacts of our work in Appendix~\ref{sec:supp-limit}~and~\ref{sec:supp-impact}, respectively, in the supplement.}


\section{Acknowledgements}
\fcom{The authors gratefully acknowledge support by the National Science Foundation under grants CCF-1937500 and CNS-2112471, and the National Institutes of Health under grant NHLBI U01HL089856.}

\clearpage
\bibliographystyle{plain}






\appendix
\newpage
\onecolumn

\section{Checklist}


\begin{enumerate}

\item For all authors...
\begin{enumerate}
  \item Do the main claims made in the abstract and introduction accurately reflect the paper's contributions and scope?
    \answerYes{}
  \item Did you describe the limitations of your work?
    \answerYes{See Section~\ref{sec:conclusion} and Appendix~\ref{sec:supp-limit} in the supplement.}
  \item Did you discuss any potential negative societal impacts of your work?
    \answerYes{See Section~\ref{sec:conclusion} and Appendix~\ref{sec:supp-impact} in the supplement.}
  \item Have you read the ethics review guidelines and ensured that your paper conforms to them?
    \answerYes{}
\end{enumerate}

\item If you are including theoretical results...
\begin{enumerate}
  \item Did you state the full set of assumptions of all theoretical results?
    \answerYes{See Section~\ref{sec:HBAR_theorem}, Assumption~\ref{asm:cont}~and~\ref{asm:universal}.}
	\item Did you include complete proofs of all theoretical results?
    \answerYes{See Appendix~\ref{sec:supp-proof-thm1} and~\ref{sec:supp-proof-thm2} in the supplement.}
\end{enumerate}

\item If you ran experiments...
\begin{enumerate}
  \item Did you include the code, data, and instructions needed to reproduce the main experimental results (either in the supplemental material or as a URL)?
    \answerYes{See Section~\ref{sec:experient-setup}. We provide code and instructions to reproduce the main experimental results for our proposed method in the supplement.}
  \item Did you specify all the training details (e.g., data splits, hyperparameters, how they were chosen)?
    \answerYes{See Section~\ref{sec:experient-setup} and Appendix~\ref{sec:supp-algorithms} in the supplement.}
	\item Did you report error bars (e.g., with respect to the random seed after running experiments multiple times)?
    \answerYes{See Figure~\ref{fig:metrics_versus_epochs} in the main text and Figure~\ref{fig:hsic-adv-variance} in the supplement.}
	\item Did you include the total amount of compute and the type of resources used (e.g., type of GPUs, internal cluster, or cloud provider)?
    \answerYes{See Section~\ref{sec:experient-setup}.}
\end{enumerate}

\item If you are using existing assets (e.g., code, data, models) or curating/releasing new assets...
\begin{enumerate}
  \item If your work uses existing assets, did you cite the creators?
    \answerYes{See Section~\ref{sec:experient-setup}.}
  \item Did you mention the license of the assets?
    \answerYes{See Section~\ref{sec:experient-setup}~and Appendix~\ref{sec:licensing} in the supplement.}
  \item Did you include any new assets either in the supplemental material or as a URL?
    \answerYes{We provide code for our proposed method in the supplement}.
  \item Did you discuss whether and how consent was obtained from people whose data you're using/curating?
    \answerNA{}
  \item Did you discuss whether the data you are using/curating contains personally identifiable information or offensive content?
    \answerNA{}
\end{enumerate}

\item If you used crowdsourcing or conducted research with human subjects...
\begin{enumerate}
  \item Did you include the full text of instructions given to participants and screenshots, if applicable?
    \answerNA{}
  \item Did you describe any potential participant risks, with links to Institutional Review Board (IRB) approvals, if applicable?
    \answerNA{}
  \item Did you include the estimated hourly wage paid to participants and the total amount spent on participant compensation?
    \answerNA{}
\end{enumerate}

\end{enumerate}

\section{Proof of Theorem \ref{thm:main_theorem}} \label{sec:supp-proof-thm1}

\begin{proof}

The following lemma holds:
\begin{lemma}
\label{thm:hsic_cov_1}
\citep{gretton2005measuring, greenfeld2020robust}
Let $X$, $Z$ be random variables residing in metric spaces $\mathcal{X}$, $\mathcal{Z}$, respectively. Let also $\mathcal{F}, \mathcal{G}$ be the two separable RKHSs on $\mathcal{X}, \mathcal{Z}$ induced by $k_X$ and $k_Z$, respectively. Then, the following inequality holds:
\begin{equation}
    \operatorname{HSIC}(X, Z) \geq \sup _{s \in \mathcal{F}, t \in \mathcal{G}} \operatorname{Cov}[s(X), t(Z)].
\end{equation}
\end{lemma}

Lemma \ref{thm:hsic_cov_1} shows that HSIC bounds the supremum of the covariance between any pair of functions in the RKHS, $\mathcal{F}, \mathcal{G}$. Assumption~\ref{asm:universal} states that functions in $\mathcal{F}$ and $\mathcal{G}$ are uniformly bounded by  $M_{\mathcal{F}}> 0$ and $M_{\mathcal{G}}>0 $, respectively. 
Let $\mathcal{\tilde{F}}$ and $\mathcal{\tilde{G}}$ be the restriction of $\mathcal{F}$ and $\mathcal{G}$ to functions in the unit ball of the respective RKHSs through rescaling, i.e.:
\begin{equation}\label{'eq: unit_ball'}
\begin{split}
    \mathcal{\tilde{F}} = \left\{  \frac{h}{M_{\mathcal{F}}} : h \in \mathcal{F}\right\}
    \quad \text{and} \quad
    \mathcal{\tilde{G}} = \left\{ \frac{g}{M_{\mathcal{G}}} : g \in \mathcal{G}\right\}.
\end{split}
\end{equation}
The following lemma links the covariance of the functions in the original RKHSs to their normalized version:
\begin{lemma}\label{'le:A_2'}
\citep{greenfeld2020robust}
Suppose $\mathcal{F}$ and $\mathcal{G}$ are RKHSs over $\mathcal{X}$ and $\mathcal{Z}$, s.t.
$\|s\|_{\infty} \leq M_{\mathcal{F}}$ for all $s \in \mathcal{F}$ and $\|t\|_{\infty} \leq M_{\mathcal{G}}$ for all $t \in \mathcal{G}$.
Then the following holds:
\begin{align}
\begin{split}
    \sup _{s \in \mathcal{F}, t \in \mathcal{G}} \operatorname{Cov}[s(X), &t(Z)] = 
     M_{\mathcal{F}}M_{\mathcal{G}} \sup _{s \in \tilde{\mathcal{F}}, t \in \tilde{\mathcal{G}}} \operatorname{Cov}[s(X), t(Z)].
\end{split}
\end{align}
\end{lemma}

 For simplicity in notation, we define the following sets containing functions that satisfy Assumption \ref{asm:cont}:
\begin{equation}\label{'eq: continuous'}
\begin{split}
    C_{b}(\mathcal{X}) = \left\{h \in C(\mathcal{X}) : ||h||_{\infty}  \leq M_{\mathcal{X}}\right\} \quad \text{and} \quad
    C_{b}(\mathcal{Z}) = \left\{g \in C(\mathcal{Z}) : ||g||_{\infty}  \leq M_{\mathcal{Z}}\right\}.
\end{split}
\end{equation}

In Assumption \ref{asm:universal}, we mention that functions in $\mathcal{F}$ and $\mathcal{G}$ may require appropriate rescaling to keep the universality of corresponding kernels. To make the rescaling explicit, we define the following \textit{rescaled} RKHSs:
\begin{equation}\label{'eq: rescaled_rkhs'}
\begin{split}
    \mathcal{\hat{F}} = \left\{\frac{M_{\mathcal{X}}}{M_{\mathcal{F}}} \cdot h : h\in \mathcal{F} \right\} \quad \text{and} \quad
    \mathcal{\hat{G}} = \left\{\frac{M_{\mathcal{Z}}}{M_{\mathcal{G}}} \cdot g : g\in \mathcal{G} \right\}.
\end{split}
\end{equation}
This rescaling ensures that $||\hat{h}||_{\infty} \leq M_\mathcal{X}$ for every $\hat{h}\in \mathcal{\hat{F}}$. Similarly, $||\hat{g}||_{\infty} \leq M_\mathcal{Z}$ for every $\hat{g}\in \mathcal{\hat{G}}$. 



We also want to prove $\mathcal{F}$ is convex. Given $f,g\in \mathcal{F}$, we need to show for all $0\leq \alpha \leq 1$, the function $\alpha f + (1-\alpha) g \in \mathcal{F}$.
As linear summation of RKHS functions is in the RKHS, we just need to check that $|| \alpha f + (1-\alpha) g||_{\infty} \leq M_{\mathcal{F}}$; indeed:
\begin{equation}
    || \alpha f + (1-\alpha) g||_{\infty} \leq \alpha ||f||_{\infty} + (1-\alpha) ||g||_{\infty} \leq  \alpha M_{\mathcal{F}} + (1-\alpha) M_{\mathcal{F}}
\end{equation}
We thus conclude that the bounded RKHS $\mathcal{F}$ is indeed convex. Hence any rescaling of the function, as long as it has a norm less than  $M_{\mathcal{F}} $, remains inside $\mathcal{F}$.

Indeed, the following lemma holds:


\begin{lemma} \label{'lem: hat_universal'}
 If $\mathcal{F},\mathcal{G}$ are universal with respect to $C_{b}(\mathcal{X}),C_{b}(\mathcal{Z})$, then:
 
 \begin{equation} \label{'eq:restricted universal'}
 \mathcal{\hat{F}} = C_{b}(\mathcal{X})\quad\text{and}\quad \mathcal{\hat{G}} = C_{b}(\mathcal{Z}).
 \end{equation}

\end{lemma}

\begin{proof}

We prove this by first showing $C_{b}(\mathcal{X}) \subseteq \mathcal{\hat{F}}$ and then $\mathcal{\hat{F}} \subseteq C_{b}(\mathcal{X})$, which leads to equality of the sets.

\begin{itemize}
    \item  $C_{b}(\mathcal{X}) \subseteq \mathcal{\hat{F}}$: 
    For all $h \in C_{b}(\mathcal{X})$, we show $h\in\mathcal{\hat{F}}$. Based on the definition of $C_{b}(\mathcal{X})$ in~\eqref{'eq: continuous'}, we know $\|h\|_{\infty} \leq M_{\mathcal{X}}$. From universality stated in Assumption~\ref{asm:universal}, $h \in \mathcal{F}$. Let $g = \frac{M_{\mathcal{F}}}{M_{\mathcal{X}}}h$. Then $\|g\|_{\infty} = \|\frac{M_{\mathcal{F}}}{M_{\mathcal{X}}}h\|_{\infty} = \frac{M_{\mathcal{F}}}{M_{\mathcal{X}}}||h||_{\infty} \leq  M_{\mathcal{F}}$. Based on the convexity of $\mathcal{F}$, $g \in \mathcal{F}$. We rescale every function in $\mathcal{F}$ by $\frac{M_{\mathcal{X}}}{M_{\mathcal{F}}}$ to form $\mathcal{\hat{F}}$, so $\frac{M_{\mathcal{X}}}{M_{\mathcal{F}}} g = \frac{M_{\mathcal{X}}}{M_{\mathcal{F}}} \frac{M_{\mathcal{F}}}{M_{\mathcal{X}}} h = h \in \mathcal{\hat{F}}$.

    \item $\mathcal{\hat{F}} \subseteq C_{b}(\mathcal{X})$:
    On the other hand, for all $h \in \mathcal{\hat{F}}$, $h$ is continuous and bounded by $M_{\mathcal{X}}$. So  based on the definition of $C_{b}(\mathcal{X})$ in~\eqref{'eq: continuous'}, $h\in C_{b}(\mathcal{X})$. Thus, $ \mathcal{\hat{F}} \subseteq C_{b}(\mathcal{X})$. \\

\end{itemize}

Having both side of the inclusion we conclude that $\mathcal{\hat{F}}= C_{b}(\mathcal{X})$. One can prove $\mathcal{\hat{G}}= C_{b}(\mathcal{Z})$ similarly.

\end{proof}





Applying the universality of kernels from Assumption $\ref{asm:universal}$ we can prove the following lemma:

\begin{lemma}\label{'le:equality in sup'}
Let $X$, $Z$ be random variables residing in metric spaces $\mathcal{X}$, $\mathcal{Z}$ with separable RKHSs $\mathcal{F}$, $\mathcal{G}$ induced by kernel functions $k_X$ and $k_Z$, respectively, for which Assumption~\ref{asm:universal} holds. Let $\mathcal{\hat{F}}$ and $\mathcal{\hat{G}}$ be the rescaled RKHSs  defined in~\eqref{'eq: rescaled_rkhs'}. Then: 
\begin{equation}
  \frac{M_{\mathcal{X}}M_{\mathcal{Z}}}{ M_{\mathcal{F}}M_{\mathcal{G}}}  \sup _{s \in \mathcal{F}, t \in \mathcal{G}} \operatorname{Cov}[s(X), t(Z)]= \sup _{s \in \mathcal{\hat{F}}, t \in \mathcal{\hat{G}}} \operatorname{Cov}[s(X), t(Z)] = \sup _{s \in C_{b}(\mathcal{X}), t \in C_{b}(\mathcal{Z})} \operatorname{Cov}[s(X), t(Z)],
\end{equation} 
where $C_{b}(\mathcal{X}), C_{b}(\mathcal{Z})$ are defined in \eqref{'eq: continuous'}.
\end{lemma}

\begin{proof}

The right equality of Lemma  \ref{'le:equality in sup'} immediately follows by Lemma~\ref{'lem: hat_universal'}:
\begin{equation}
    \sup _{s \in \mathcal{\hat{F}}, t \in \mathcal{\hat{G}}} \operatorname{Cov}[s(X), t(Z)] = \sup _{s \in C_{b}(\mathcal{X}), t \in C_{b}(\mathcal{Z})} \operatorname{Cov}[s(X), t(Z)].
\end{equation}

Applying Lemma  \ref{'le:A_2'} on $\mathcal{F}, \mathcal{G}, \mathcal{\tilde{F}},\mathcal{\tilde{G}}$, we have: 
\begin{equation}\label{'eq:lemma3_result_1'}
    \sup _{s \in \mathcal{F}, t \in \mathcal{G}} \operatorname{Cov}[s(X), t(Z)] = 
     M_{\mathcal{F}}M_{\mathcal{G}} \sup _{s \in \tilde{\mathcal{F}}, t \in \tilde{\mathcal{G}}} \operatorname{Cov}[s(X), t(Z)].
\end{equation}

Note that from \eqref{'eq: rescaled_rkhs'} and \eqref{'eq: unit_ball'}, we have that the corresponding normalized space for $\mathcal{\hat{F}}$ is:
\begin{equation} \label{eq:21}
    \begin{split}
        \left\{  \frac{h}{M_{\mathcal{X}}} : h \in \mathcal{\hat{F}}\right\} = \left\{  \frac{M_{\mathcal{X}}}{M_{\mathcal{F}}}\frac{h}{M_{\mathcal{X}}}  : h \in \mathcal{{F}}\right\} = \left\{  \frac{h}{M_{\mathcal{F}}} : h \in \mathcal{F}\right\} = \mathcal{\tilde{F}}.
    \end{split}
\end{equation}
Similarly, the normalized space for $\mathcal{\hat{G}}$ is:
\begin{equation} \label{eq:22}
    \begin{split}
        \left\{  \frac{g}{M_{\mathcal{Z}}} : g \in \mathcal{\hat{G}}\right\}  = \left\{  \frac{g}{M_{\mathcal{G}}} : g \in \mathcal{G}\right\}= \mathcal{\tilde{G}}.
    \end{split}
\end{equation}
Equation~\eqref{eq:21} implies that the normalized space induced from $\mathcal{\hat{F}}$ coincides with the normalized space induced from $\mathcal{{F}}$. Similarly, Equation \eqref{eq:22} implies the normalized spaces for $\mathcal{G}$ and $\mathcal{\hat{G}}$ also coincide. Moreover, for all $\hat{h}\in \mathcal{\hat{F}}$, $||\hat{h}||_{\infty} \leq M_\mathcal{X}$ and for all $\hat{g}\in \mathcal{\hat{G}}$, $||\hat{g}||_{\infty} \leq M_\mathcal{Z}$. Hence, applying Lemma~\ref{'le:A_2'} on $\mathcal{\hat{F}}, \mathcal{\hat{G}}, \mathcal{\tilde{F}},\mathcal{\tilde{G}}$, we have:  
\begin{equation}\label{'eq:lemma3_result_2'}
    \sup _{s \in \mathcal{\hat{F}}, t \in \mathcal{\hat{G}}} \operatorname{Cov}[s(X), t(Z)] = 
     M_{\mathcal{X}}M_{\mathcal{Z}} \sup _{s \in \tilde{\mathcal{F}}, t \in \tilde{\mathcal{G}}} \operatorname{Cov}[s(X), t(Z)].
\end{equation}

By dividing Equation~\eqref{'eq:lemma3_result_1'} and \eqref{'eq:lemma3_result_2'}, we prove the left part of Lemma \ref{'le:equality in sup'}:
\begin{equation}
  \frac{M_{\mathcal{X}}M_{\mathcal{Z}}}{ M_{\mathcal{F}}M_{\mathcal{G}}}  \sup _{s \in \mathcal{F}, t \in \mathcal{G}} \operatorname{Cov}[s(X), t(Z)]=  \sup _{s \in \mathcal{\hat{F}}, t \in \mathcal{\hat{G}}} \operatorname{Cov}[s(X), t(Z)].
\end{equation} 
\end{proof}

By combining Theorem \ref{thm:hsic_cov_1} and Lemma \ref{'le:equality in sup'}, we have the following result: 
\begin{equation} \label{eq:last_second}
    \begin{split}
        \frac{M_{\mathcal{X}}M_{\mathcal{Z}}}{ M_{\mathcal{F}}M_{\mathcal{G}}}\operatorname{HSIC}(X, Z) \geq
     \sup _{s \in C_{b}(\mathcal{X}), t \in C_{b}(\mathcal{Z})} \operatorname{Cov}[s(X), t(Z)].
    \end{split}
\end{equation}

Recall that $h_{\theta}$ is a neural network from $\mathcal{X}$ to $\mathcal{Y}$, such that it can be written as composition of $g\circ f$, where $f:\mathcal{X}\to \mathcal{Z}$ and $g: \mathcal{Z}\to \mathcal{Y}$. Moreover, $h_{\theta} \in C_{b}(\mathcal{X})$ and  $g \in C_{b}(\mathcal{Z})$. Using the fact that the supremum on a subset of a set is smaller or equal than the supremum on the whole set, we conclude that: 
\begin{equation}
    \begin{split}
    \label{eq:hsic_var}
    \frac{M_{\mathcal{X}}M_{\mathcal{Z}}}{ M_{\mathcal{F}}M_{\mathcal{G}}}\operatorname{HSIC}(X, Z)  &\geq  \sup_{\theta} \operatorname{Cov}[h_{\theta}(X), g(Z))] \\
    &=  \sup_{\theta}
    \operatorname{Cov}[h_{\theta}(X) , g\circ f (X)] \\ 
    &=  \sup_{\theta}\operatorname{Var}[h_{\theta}(X)].
    \end{split}
\end{equation}
\end{proof}

\section{Proof of Theorem 2} \label{sec:supp-proof-thm2}



\begin{proof}
\eqcom
Let $t_i:\reals^{\dx}\to\reals$,  $i = 1, 2, ..., \dx$ be the following truncation functions:
\begin{align} \label{eq:truncation}
    \begin{split}
        t_i(X) = \begin{cases} 
          -R, & \text{if}~X_i < -R, \\
          X_i, & \text{if}~-R \leq X_i \leq R, \\
          R, &\text{if}~ X_i > R.
       \end{cases} 
    \end{split}
\end{align}

where $0 < R < \infty$ and $X_i$ is the $i$-th dimension of $X$. Functions $t_i$ are continous and bounded  in $\mathcal{X}$, and
\begin{equation}
    t_i \in C_{b'}(\mathcal{X}), \quad \text{where}\quad C_{b'}(\mathcal{X})= \{t \in C(\mathcal{X}): \|t\|_\infty \leq R\}
\end{equation}
Moreover,  $g$ satisfies Assumptions~\ref{asm:cont} and~\ref{asm:universal}.  
Similar to the proof of Theorem~\ref{thm:main_theorem}, by combining Theorem \ref{thm:hsic_cov_1} and Lemma \ref{'le:equality in sup'}, we have that:
\begin{equation} \label{eq:main_for_thm2}
    \begin{split}
         \frac{R M_{\mathcal{Z}}}{ M_{\mathcal{F}}M_{\mathcal{G}}}\operatorname{HSIC}(X, Z) 
     &\geq \sup_{t \in C_{b'(\mathcal{X})},\ g \in C_{b}(\mathcal{Z})} \operatorname{Cov}[t (X), g(Z)] \\
     &\geq \operatorname{Cov}[t_i (X), h_{\theta}(X)], \quad i = 1, \ldots, \dx.
    \end{split}
\end{equation}

Moreover, the following lemma holds:
\begin{lemma} \label{lemma: cov_diff}
Let $X \sim \mathcal{N}(0, \sigma^2 \mathbf{I})$ and $t_i(X)$ defined by~\eqref{eq:truncation}. For all $h_\theta$ that satisfy Assumption~\ref{asm:cont},  we have:
\begin{align}
        \operatorname{Cov}[X_i, h_{\theta}(X)] - \operatorname{Cov}[t_i(X), h_{\theta}(X)] \leq \frac{2 M_\mathcal{X} \sigma}{\sqrt{2 \pi}} \exp(-\frac{R^2}{2\sigma^2}), \quad \text{for all}~i = 1, 2, \ldots, \dx.
\end{align}
\end{lemma}
\begin{proof}
\begin{subequations}
\begin{align}
    \label{cov_diff_1}
    \text{LHS} &= \int_{-\infty}^{\infty} (x_i-t_i(x)) h_\theta (x) \frac{1}{\sqrt{2 \pi \sigma^2}} \exp(-\frac{x_i^2}{2\sigma^2}) dx_i \\ 
    \label{cov_diff_2}
    &=\frac{1}{\sqrt{2 \pi \sigma^2}} \left( \int_{-\infty}^{-R} (x_i+R) h_\theta (x) \exp(-\frac{x_i^2}{2\sigma^2}) dx_i + \int_{R}^{\infty} (x_i-R)
    h_\theta (x) \exp(-\frac{x_i^2}{2\sigma^2}) dx_i \right) \\
    \label{cov_diff_3}
    &\leq \frac{2 M_\mathcal{X}}{\sqrt{2 \pi \sigma^2}} \int_{R}^{\infty} (x_i-R) \exp(-\frac{x_i^2}{2\sigma^2}) dx_i \\
    \label{cov_diff_4}
    &= \frac{2 M_\mathcal{X}}{\sqrt{2 \pi \sigma^2}} \int_{R}^{\infty} x_i \exp(-\frac{x_i^2}{2\sigma^2}) dx_i - \frac{2 M_\mathcal{X} R}{\sqrt{2 \pi \sigma^2}} \int_{R}^{\infty} \exp(-\frac{x_i^2}{2\sigma^2}) dx_i \\
    \label{cov_diff_5}
    &\leq \frac{2 M_\mathcal{X}}{\sqrt{2 \pi \sigma^2}} \int_{R}^{\infty} x_i \exp(-\frac{x_i^2}{2\sigma^2}) dx_i \\
    \label{cov_diff_6}
    &=  \frac{2 M_\mathcal{X} \sigma}{\sqrt{2 \pi}} \exp(-\frac{R^2}{2\sigma^2}),
\end{align}
where \eqref{cov_diff_1}, \eqref{cov_diff_2}, \eqref{cov_diff_4}, \eqref{cov_diff_6} are direct results from definition or simple calculation, \eqref{cov_diff_3} comes from the fact that $M_\mathcal{X} = \max \| h_\theta(X) \|_\infty $ and the symmetry of two integrals, and \eqref{cov_diff_5} is due to the non-negativity of the probability density function.
\end{subequations} 
\end{proof}
Combining Lemma~\ref{lemma: cov_diff} with~\eqref{eq:main_for_thm2}, we have the following result:
\begin{equation} \label{eq:thm2_intermediate}
    \frac{R M_{\mathcal{Z}}}{ M_{\mathcal{F}}M_{\mathcal{G}}}\operatorname{HSIC}(X, Z) + \frac{2 M_\mathcal{X} \sigma}{\sqrt{2 \pi}} \exp(-\frac{R^2}{2\sigma^2}) \geq \operatorname{Cov}[ X_i, h_{\theta}(X)], \quad \text{for all}~i = 1, \ldots, \dx.
\end{equation}
We can further bridge HSIC to adversarial robustness directly by taking advantage of the following lemma:

\begin{lemma}[Stein's Identity \citep{liu1994siegel}]
\label{thm:stein}
Let $X = (X_1, X_2,  \ldots X_{\dx})$ be multivariate normally distributed with arbitrary mean vector $\mu$ and covariance matrix $\Sigma$. For any function $h(x_1, \ldots, x_{\dx})$ such that $\frac{\partial h}{\partial x_i}$ exists almost everywhere and $\E |\frac{\partial}{\partial x_i}| < \infty$, $i=1, \ldots, \dx$, we write 
$\nabla h(X) = (\frac{\partial h(X)}{\partial x_1}, \ldots, \frac{\partial h(X)}{\partial x_{\dx}})^\top$. Then the following identity is true:
\begin{equation}
    \operatorname{Cov}[X, h(X)]=\Sigma E[\nabla h(X)].
\end{equation}
Specifically,
\begin{equation}
    \operatorname{Cov}\left[X_{1}, h\left(X_{1}, \ldots, X_{\dx}\right)\right]=\sum_{i=1}^{\dx} \operatorname{Cov}\left(X_{1}, X_{i}\right) E\left[\frac{\partial}{\partial x_{i}} h\left(X_{1}, \ldots, X_{\dx}\right)\right]
\end{equation}
\end{lemma}


Given that $X \sim \mathcal{N}(0, \sigma^2 \mathbf{I})$,  Lemma~\ref{thm:stein} implies:
\begin{equation} \label{eq:stein}
    \operatorname{Cov}\left[X_{i}, h_\theta\left(X\right)\right]= \sigma^2 \E\left[\frac{\partial}{\partial x_{i}} h_\theta\left(X\right)\right].
\end{equation}


Combining~\eqref{eq:thm2_intermediate} and~\eqref{eq:stein}, we have:
\begin{align}
    \frac{R M_{\mathcal{Z}}}{ M_{\mathcal{F}}M_{\mathcal{G}}}\operatorname{HSIC}(X, Z) + \frac{2 M_\mathcal{X} \sigma}{\sqrt{2 \pi}} \exp(-\frac{R^2}{2\sigma^2}) &\geq \sigma^2 \E\left[\frac{\partial}{\partial x_{k}} h_\theta\left(X\right)\right].
\end{align}

Note that a  similar derivation could be repeated exactly by replacing $h_\theta (X)$ with $-h_\theta (X)$. Thus, for every $i = 1, 2, \ldots, \dx$, we have:
\begin{align} \label{eq:abs_partial_bound}
     \frac{R M_{\mathcal{Z}}}{ M_{\mathcal{F}}M_{\mathcal{G}}}\operatorname{HSIC}(X, Z) + \frac{2 M_\mathcal{X} \sigma}{\sqrt{2 \pi}} \exp(-\frac{R^2}{2\sigma^2}) &\geq  \sigma^2 \E\left[ \left| \frac{\partial}{\partial x_{i}} h_{\theta}\left(X\right) \right |\right]. 
\end{align}

Summing up both sides in \eqref{eq:abs_partial_bound} for $i = 1, 2, \ldots, \dx$, we have:
\begin{align} \label{eq:total_abs_partial_bound}
    \frac{\dx R M_{\mathcal{Z}}}{ M_{\mathcal{F}}M_{\mathcal{G}}}\operatorname{HSIC}(X, Z) + \frac{2 \dx M_\mathcal{X} \sigma}{\sqrt{2 \pi}} \exp(-\frac{R^2}{2\sigma^2}) &\geq  \sigma^2 \E\left[ \sum_{i=1}^{\dx} \left| \frac{\partial}{\partial x_{i}} h_{\theta}\left(X\right) \right |\right] .
\end{align}

On the other hand, for   $\delta \in \mathcal{S}_r$, by Taylor's theorem:
\begin{subequations} \label{eq:adv_difference_bound}
\begin{align}
\label{eq:taylor}
    \E [|h_\theta(X+\delta) - h_\theta(X)|] &\leq \E [|\delta^\top \nabla_{X} h_\theta (X)|] + o(r) \\ 
\label{eq:holder}
    &\leq \E \left[ \| \delta \|_{\infty} \|\nabla_{X} h_\theta (X)\|_1 \right] + o(r)  \\ 
\label{eq:triangle}
    &\leq r \E \left[ \sum_{i=1}^{\dx} \left| \frac{\partial}{\partial x_{i}} h_{\theta}\left(X\right) \right | \right] + o(r), 
\end{align}
\end{subequations}
where~\eqref{eq:holder} is implied by H\"older's inequality, and~\eqref{eq:triangle} is implied by the triangle inequality.

Combining \eqref{eq:total_abs_partial_bound} and \eqref{eq:adv_difference_bound}, we have:
\begin{equation}
     \frac{ r \dx R M_{\mathcal{Z}}}{ \sigma^2 M_{\mathcal{F}}M_{\mathcal{G}}}\operatorname{HSIC}(X, Z) + \frac{2 r \dx M_\mathcal{X}}{\sqrt{2 \pi} \sigma} \exp(-\frac{R^2}{2\sigma^2}) + o(r) \geq \E [|h_\theta(X+\delta) - h_\theta(X)|].
\end{equation}

Let $R = \sigma \sqrt{-2\log o(1)}$ where, here, $o(1)$ stands for an arbitrary function $w:\reals\to\reals$ s.t. \begin{align}\lim_{r\to 0} w(r) = 0.\end{align}
Then, we have $\frac{2 r \dx M_\mathcal{X}}{\sqrt{2 \pi} \sigma} \exp(-\frac{R^2}{2\sigma^2}) = o(r)$, because:
\begin{align}
    \begin{split}
        \lim_{r \to 0} \frac{2 r \dx M_\mathcal{X}}{\sqrt{2 \pi} \sigma} \exp(-\frac{R^2}{2\sigma^2}) / r &= \lim_{r \to 0} \frac{2 \dx M_\mathcal{X}}{\sqrt{2 \pi} \sigma} \exp( \log o(1)) \\
        &= \lim_{r \to 0} \frac{2 \dx M_\mathcal{X}}{\sqrt{2 \pi} \sigma} o(1) \\ &= 0
    \end{split}
\end{align}
Thus, we conclude that:
\begin{align}
     \frac{ r \sqrt{-2 \log o(1)} \dx M_{\mathcal{Z}}}{ \sigma M_{\mathcal{F}}M_{\mathcal{G}}}\operatorname{HSIC}(X, Z) + o(r) \geq \E [|h_\theta(X+\delta) - h_\theta(X)|].
\end{align}
\end{proof}



\section{Licensing of Existing Assets} \label{sec:licensing}
We provide the licensing information of each existing asset below:

\textbf{Datasets.}
\begin{packeditemize}
    \item \emph{MNIST} \emph{mnist} is licensed under the Creative Commons Attribution-Share Alike 3.0 license.
    \item \emph{CIFAR-10}~and~\emph{CIFAR-100}~\citep{cifar} are licensed under the MIT license.
\end{packeditemize}
\textbf{Models.}
\begin{packeditemize}
    \item The implementations of \emph{LeNet} \citep{madry2017towards} and \emph{ResNet-18} \citep{he2016deep} in our paper are licensed under BSD 3-Clause License. 
    \item The implementation of \emph{WideResNet-28-10} \citep{wideresnet} is licensed under the MIT license.
\end{packeditemize}
\textbf{Algorithms.}
\begin{packeditemize}
    \item The implementations of \emph{SWHB} \citep{ma2020hsic}, \emph{PGD} \citep{madry2017towards}, \emph{TRADES} \citep{zhang2019theoretically} are licensed under the MIT license.  
    \item The implementation of \emph{VIB} \citep{alemi2016deep} is licensed under the Apache License 2.0. 
    \item There are no licenses for \emph{MART}~\citep{wang2019improving} and \emph{XIC} \citep{greenfeld2020robust}.
\end{packeditemize}

\textbf{Adversarial Attacks.} The implementations of \emph{FGSM}~\citep{goodfellow2014explaining}, \emph{PGD}~\citep{madry2017towards}, \emph{CW}~\citep{cw} and \emph{AutoAttack}~\citep{autoattack} are all licensed under the MIT license.

\section{Algorithm Details and Hyperparameter Tuning} \label{sec:supp-algorithms}
Non-adversarial learning, information bottleneck based methds:
\begin{packeditemize}
    \item \emph{Cross-Entropy (CE)}, which includes only loss $\loss$.
    \item \emph{Stage-Wise HSIC Bottleneck (SWHB)} \citep{ma2020hsic}: This is the original HSIC bottleneck. It does not include full backpropagation over the HSIC objective: early layers are fixed stage-wise, and gradients are computed only for the current layer.
    \item \emph{XIC} \citep{greenfeld2020robust}: To enhance generalization over distributional shifts, this penalty includes inputs and residuals (i.e., $\HSIC(X, Y-h(X))$).
    \item \emph{Variational Information Bottleneck (VIB)} \citep{alemi2016deep}: this is a variational autoencoder that includes a mutual information bottleneck penalty. 
\end{packeditemize}

Adversarial learning methods:
\begin{packeditemize}
    \item \emph{Projected Gradient Descent (PGD)} \citep{madry2017towards}: This optimizes $\aloss$, given by \eqref{eq:adv_robustness} via projected gradient ascent over $\mathcal{S}_r$  .
    \item \emph{TRADES} \citep{zhang2019theoretically}: This uses a regularization term that minimizes the difference between the predictions of natural and adversarial examples to get a smooth decision boundary.
    \item \emph{MART} \citep{wang2019improving}: Compared to TRADES, MART pays more attention to adversarial examples from misclassified natural examples and add a KL-divergence term between natural and adversarial examples to the binary cross-entropy loss.
\end{packeditemize}
We use code provided by authors, including the recommended hyperparameter settings and tuning strategies.
In both SWHB and \hb, we apply Gaussian kernels for $X$ and $Z$ and a linear kernel for $Y$. For Gaussian kernels, we set $\sigma=5\sqrt{d}$, where $d$ is the dimension of the corresponding random variable. 

We report all tuning parameters in Table~\ref{tab:adv-competing-params}.
In particular, we report the parameter settings on the 4-layer LeNet \citep{madry2017towards} for MNIST, ResNet-18 \citep{he2016deep} and WideResNet-28-10 \citep{wideresnet} for CIFAR-10, and WideResNet-28-10 \citep{wideresnet} for CIFAR-100 with the basic \hb and when combining \hb with state-of-the-art (i.e., PGD, TRADES, MART) adversarial learning.

For \hb, to make a fair comparison with SWHB \citep{ma2020hsic}, we build our code, along with the implementation of PGD and PGD+\hb, upon their framework. When combining \hb with other state-of-the-art adversarial learning (i.e., TRADES and MART), we add our \hb implemention to the MART framework and use recommended hyperparameter settings/tuning strategies from MART and TRADES. To make a fair comparison, we use the same network architectures among all methods with the same random weight initialization and report last epoch results.

\begin{table}[hbt!]
    \centering
    \setlength{\extrarowheight}{.2em}
    \setlength{\tabcolsep}{5pt}
    \small
    \caption{Parameter Summary for MNIST, CIFAR-10, and CIFAR-100. $\lambda_x$ and $\lambda_y$ are balancing hyperparameters for \hb; $\lambda$ is balancing hyperparameter for TRADES and MART.}
    \vspace{2pt}
    \label{tab:adv-competing-params}
    \resizebox{1\textwidth}{!}{
    \begin{tabular}{||c || c || c | c c || c c | c c ||}
        \hline
        Dataset & param. & HBaR & PGD & PGD+HBaR & TRADES & TRADES+HBaR & MART & MART+HBaR \\
        \hline
        \hline
        \multirow{8}{*}{MNIST}
        & $\lambda_x$  & 1  & - & 0.003 & - & 0.001 & - & 0.001 \\
        & $\lambda_y$  & 50 & - & 0.001 & - & 0.005  & - & 0.005 \\
        \cline{3-9}
        & $\lambda$& \multicolumn{3}{c||}{-} & 5 & 5 & 5 & 5 \\
        \cline{3-9}
        & batch size & \multicolumn{3}{c||}{256} & \multicolumn{4}{c||}{256} \\
        & optimizer & \multicolumn{3}{c||}{adam} & \multicolumn{4}{c||}{sgd} \\
        & learning rate & \multicolumn{3}{c||}{0.0001} & \multicolumn{4}{c||}{0.01} \\
        & lr scheduler  & \multicolumn{3}{c||}{divided by 2 at the 65-th and 90-th epoch} & \multicolumn{4}{c||}{divided by 10 at the 20-th and 40-th epoch} \\
        & \# epochs & \multicolumn{3}{c||}{100} & \multicolumn{4}{c||}{50} \\
        \cline{3-9}
        
        \hline
        \hline
        \multirow{8}{*}{CIFAR-10/100}
        & $\lambda_x$  & 0.006 & - & 0.0005 & - & 0.0001 & - & 0.0001 \\
        & $\lambda_y$  & 0.05 & - & 0.005  & - & 0.0005  & - & 0.0005 \\
        \cline{3-9}
        & $\lambda$& \multicolumn{3}{c||}{-} & 5 & 5 & 5 & 5 \\
        \cline{3-9}
        & batch size & \multicolumn{3}{c||}{128} & \multicolumn{4}{c||}{128} \\
        & optimizer & \multicolumn{3}{c||}{adam} & \multicolumn{4}{c||}{sgd} \\
        & learning rate & \multicolumn{3}{c||}{0.01} & \multicolumn{4}{c||}{0.01} \\
        & lr scheduler  & \multicolumn{3}{c||}{cosine annealing} & \multicolumn{4}{c||}{divided by 10 at the 75-th and 90-th epoch} \\
        \cline{3-9}
        & \# epochs & 300 & \multicolumn{2}{c||}{95} & \multicolumn{4}{c||}{95} \\
        
        \hline
    \end{tabular}}
\end{table}

\newpage
\section{Sensitivity of Regularization Hyperparameters $\lambda_x$ and $\lambda_y$}\label{sec:supp-ablation}
\fcom{We provide a comprehensive ablation study on the sensitivity of $\lambda_x$ and $\lambda_y$ on MNIST and CIFAR-10 dataset with (Table \ref{tab:mnist-weight-adv} and \ref{tab:cifar-weight-adv}) and without (Table \ref{tab:mnist-weight} and 
\ref{tab:cifar-weight}) adversarial training. As a conclusion, (a) we set the weight of cross-entropy loss as $1$, and empirically set $\lambda_x$ and $\lambda_y$ according to the performance on a validation set. (b) For MNIST with adversarial training, we empirically discover that $\lambda_x : \lambda_y$ ranging around $5:1$ provides better performance; for MNIST without adversarial training, $\lambda_x=1$ and $\lambda_y=50$, inspired by SWHB (Ma et al., 2020), provide the best performance. (c) for CIFAR-10 (and CIFAR-100), with and without adversarial training, $\lambda_x : \lambda_y$ ranging from $1:5$ to $1:10$ provides better performance.}

\begin{table}[hbt!]
    \centering
    \setlength{\extrarowheight}{.2em}
    \setlength{\tabcolsep}{7pt}
    \small
    \caption{\textbf{MNIST by LeNet with adversarial training}: Ablation study on HBaR regularization hyperparameters $\lambda_x$ and $\lambda_y$ trained by \hb+TRADES over the metric of natural test accuracy (\%) and adversarial test robustness (PGD$^{40}$ and AA, \%).}
    \label{tab:mnist-weight-adv}
    \resizebox{0.5\textwidth}{!}{
    \begin{tabular}{||c c|| c c c ||}
        \hline
        $\lambda_x$ & $\lambda_y$ & Natural & PGD$^{40}$ & AA \\
        
        \hline
        \hline
        0.003 & 0.001 & 98.66 & 94.35 & 91.57 \\
        \hline
        0.003 & 0     & 98.92 & 93.05 & 90.95 \\
        0  & 0.001    & 98.86 & 91.77 & 88.21 \\
        
        \hline
        0.0025  & 0.0005  & 98.96 & 94.52 & 91.42 \\
        0.002  & 0.0005  & 98.92 & 94.13 & 91.33 \\
        0.0015  & 0.0005 & 98.93 & 94.06 & 91.43 \\
        0.001 & 0.0005 & 98.95 & 93.76 & 91.14 \\
        
        \hline
        0.001   & 0.0002 & 98.92 & 94.61 & 91.37 \\
        0.0008  & 0.0002 & 98.94 & 94.15 & 91.07 \\
        0.0006  & 0.0002 & 98.91 & 94.13 & 90.72 \\
        0.0004  & 0.0002 & 98.90 & 93.96 & 90.56 \\
        \hline
        
    \end{tabular}}
\end{table}

\begin{table}[hbt!]
    \centering
    \setlength{\extrarowheight}{.2em}
    \setlength{\tabcolsep}{7pt}
    \small
    \caption{\textbf{CIFAR-10 by WideResNet-28-10 with adversarial training}: Ablation study on HBaR regularization hyperparameters $\lambda_x$ and $\lambda_y$ trained by \hb+TRADES over the metric of natural test accuracy (\%), and adversarial test robustness (PGD$^{20}$ and AA, \%).}
    \label{tab:cifar-weight-adv}
    \resizebox{0.5\textwidth}{!}{
    \begin{tabular}{||c c|| c c c ||}
        \hline
        $\lambda_x$ & $\lambda_y$ & Natural & PGD$^{20}$ & AA \\
        
        \hline
        \hline
        0.0001 & 0.0005 & 85.61 & 56.51 & 53.53 \\
        \hline
        0.0001  & 0  & 80.19 & 49.49 & 45.33 \\
        0  & 0.0005  & 84.74 & 55.00 & 51.50 \\
        \hline
        0.001 & 0.005   & 85.70 & 55.74 & 52.78 \\
        0.0005  & 0.005  & 84.42 & 55.95 & 52.66 \\
        0.00005 & 0.0005  & 85.37 & 56.43 & 53.40 \\
        \hline
        
    \end{tabular}}
\end{table}

\begin{table}[hbt!]
    \centering
    \setlength{\extrarowheight}{.2em}
    \setlength{\tabcolsep}{7pt}
    \small
    \caption{\textbf{MNIST by LeNet without adversarial training}: Ablation study on HBaR regularization hyperparameters $\lambda_x$ and $\lambda_y$ over the metric of $\operatorname{HSIC}(X,Z_M)$, $\operatorname{HSIC}(Y,Z_M)$, natural test accuracy (\%), and adversarial test robustness (PGD$^{40}$, \%).}
    \label{tab:mnist-weight}
    \resizebox{0.52\textwidth}{!}{
    \begin{tabular}{||c c|| c c | c c ||}
        \hline
        \multirow{2}{*}{$\lambda_x$} & \multirow{2}{*}{$\lambda_y$} &
        \multicolumn{2}{c|}{HSIC} &
        \multirow{2}{*}{Natural} & \multirow{2}{*}{PGD$^{40}$}\\
        &&$(X,Z_M)$&$(Y,Z_M)$& &  \\
        \hline
        \hline
        \multicolumn{2}{||c||}{CE only}   & 45.29 & 8.73 & 99.23 & 0.00\\
        \hline
        0.0001 & 0     & 21.71 & 8.01 & 99.28 & 0.00 \\
        0.001  & 0     & 5.82  & 6.57 & 99.36 & 0.00 \\
        0.01   & 0     & 3.22  & 4.28 & 99.13 & 0.00 \\
        \hline
        0      & 1     & 56.45 & 9.00 & 98.92 & 0.00 \\
        \hline
        0.001  & 0.05  & 53.70 & 8.99 & 99.13 & 0.03 \\
        0.001  & 0.01  & 10.44 & 8.51 & 99.37 & 0.00 \\
        0.001  & 0.005 & 8.86  & 8.24 & 99.38 & 0.00 \\
        \hline
        0.01   & 0.5   & 16.13 & 8.90 & 99.14 & 5.00 \\
        0.1    & 5     & 15.81 & 8.90 & 98.96 & 7.72 \\
        1      & 50    & 15.68 & 8.89 & 98.90 & 8.33 \\
        1.1    & 55    & 15.90 & 8.88 & 98.88 & 6.99 \\
        1.2    & 60    & 15.76 & 8.89 & 98.95 & 7.24 \\
        1.5    & 75    & 15.62 & 8.89 & 98.94 & 8.23 \\
        2      & 100   & 15.41 & 8.89 & 98.91 & 7.00 \\
        \hline
        
    \end{tabular}}
\end{table}

\begin{table}[hbt!]
    \centering
    \setlength{\extrarowheight}{.2em}
    \setlength{\tabcolsep}{7pt}
    \small
    \caption{\textbf{CIFAR-10 by ResNet-18 without adversarial training}: Ablation study on HBaR regularization hyperparameters $\lambda_x$ and $\lambda_y$ over the metric of $\operatorname{HSIC}(X,Z_M)$, $\operatorname{HSIC}(Y,Z_M)$, natural test accuracy (\%), and adversarial test robustness (PGD$^{20}$, \%).}
    \label{tab:cifar-weight}
    \resizebox{0.52\textwidth}{!}{
    \begin{tabular}{||c c || c c | c c ||}
        \hline
        \multirow{2}{*}{$\lambda_x$} & \multirow{2}{*}{$\lambda_y$} &
        \multicolumn{2}{c|}{HSIC} &
        \multirow{2}{*}{Natural} & \multirow{2}{*}{PGD$^{20}$}\\
        &&$(X,Z_L)$&$(Y,Z_L)$& &  \\
        \hline
        \hline
        \multicolumn{2}{||c||}{CE only}   & 3.45  & 4.76 & 95.32 & 8.57 \\
        \hline
        0.001 & 0.05 & 43.48 & 8.93 & 95.36 & 2.91 \\
        0.002 & 0.05 & 43.15 & 8.92 & 95.55 & 2.29 \\
        0.003 & 0.05 & 41.95 & 8.90 & 95.51 & 3.98 \\
        0.004 & 0.05 & 30.12 & 8.77 & 95.45 & 5.23 \\
        0.005 & 0.05 & 11.56 & 8.45 & 95.44 & 23.73\\
        0.006 & 0.05 & 6.07  & 8.30 & 95.35 & 34.85\\
        0.007 & 0.05 & 4.81  & 8.24 & 95.13 & 15.80 \\
        0.008 & 0.05 & 4.44  & 8.21 & 95.13 & 8.43 \\
        0.009 & 0.05 & 3.96  & 8.14 & 94.70 & 10.83\\
        0.01  & 0.05 & 4.09  & 7.87 & 92.33 & 2.90 \\
        \hline
    \end{tabular}}
\end{table}



\begin{figure*}[!t]
\begin{tabular}{c c}
   \centering
   \small
   \includegraphics[width=0.49\columnwidth]{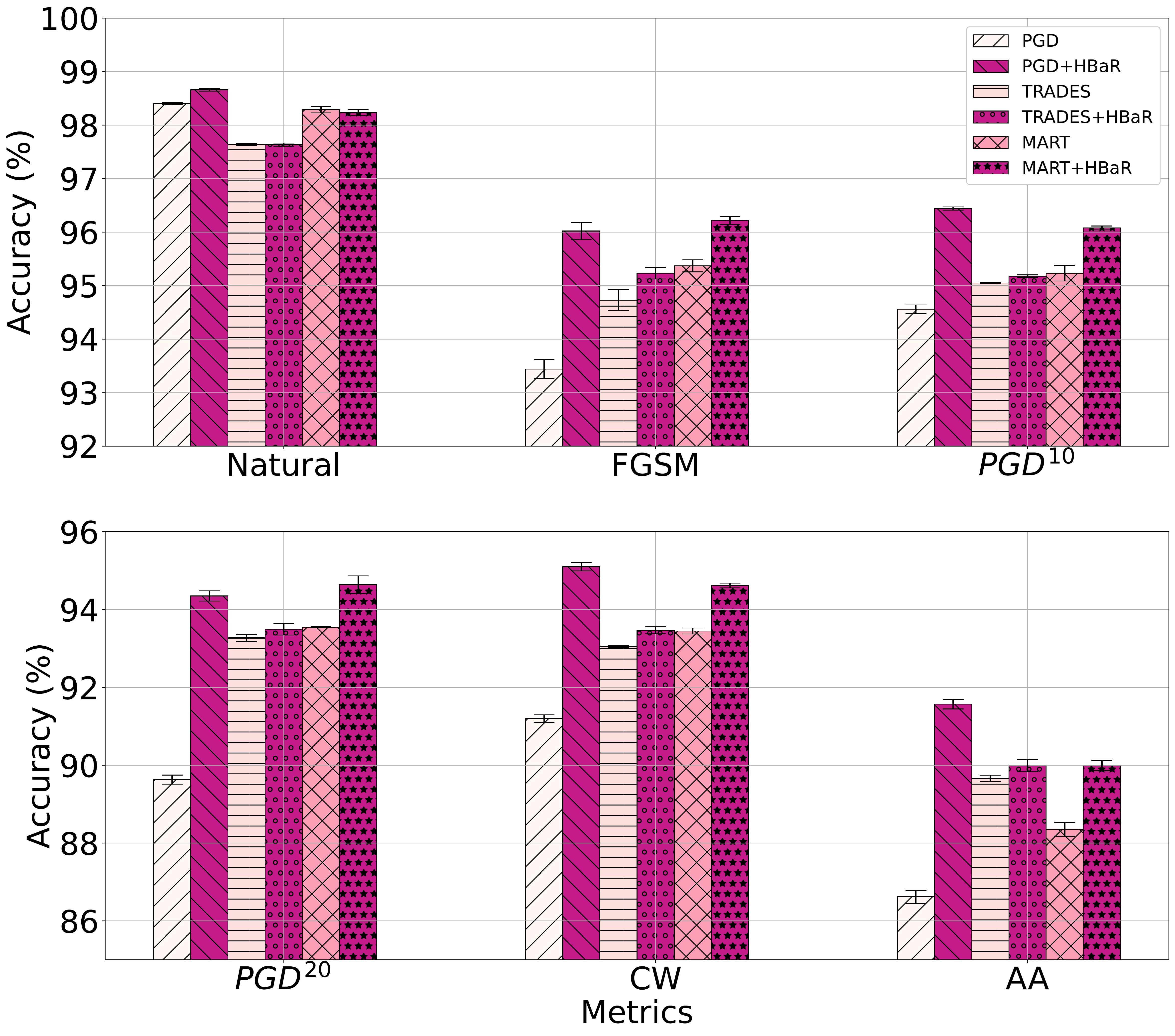} &
   \includegraphics[width=0.49\columnwidth]{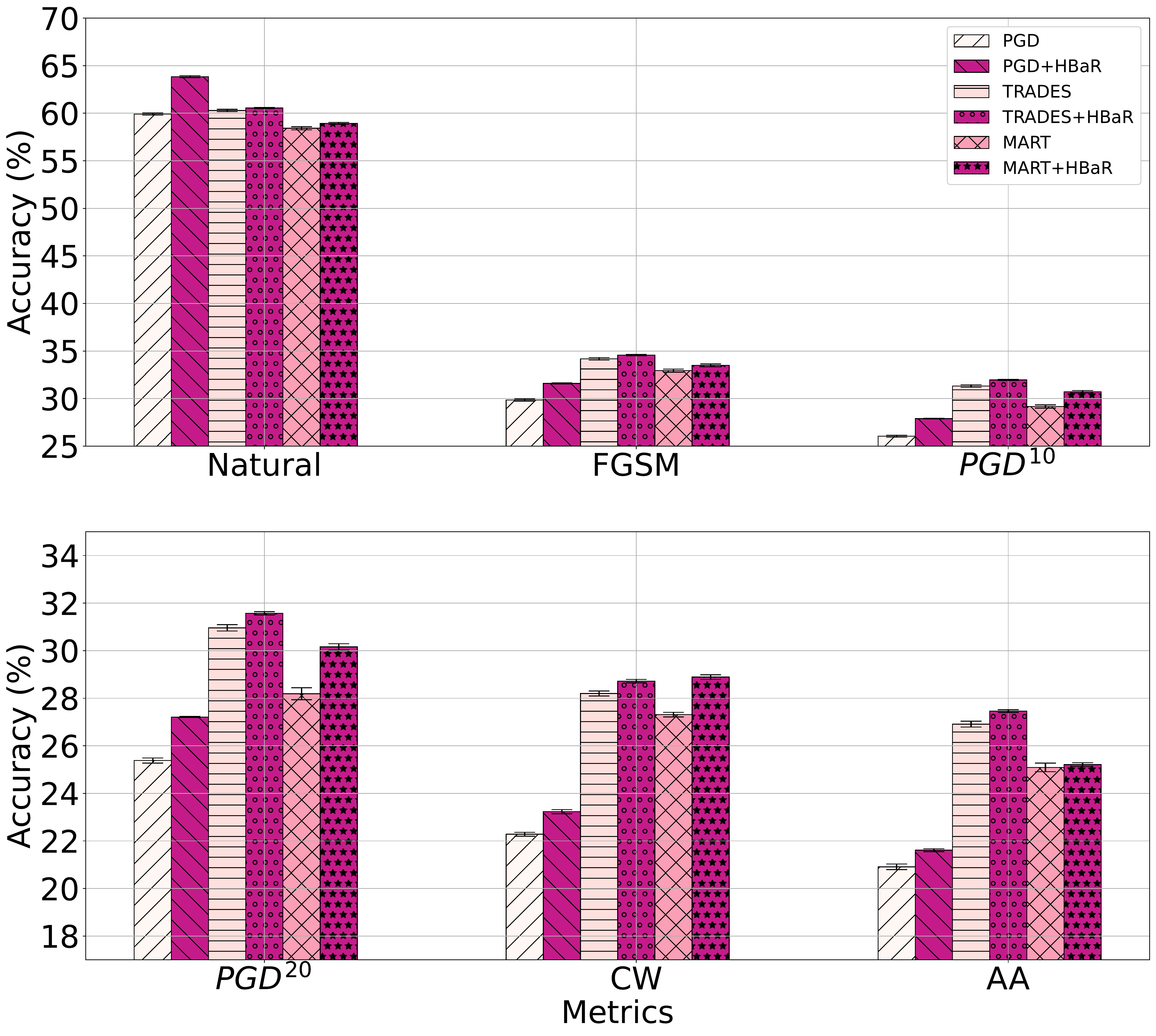} \\
   (a) MNIST by LeNet & (b) CIFAR-100 by WideResNet-28-10 \\
   \includegraphics[width=0.49\columnwidth]{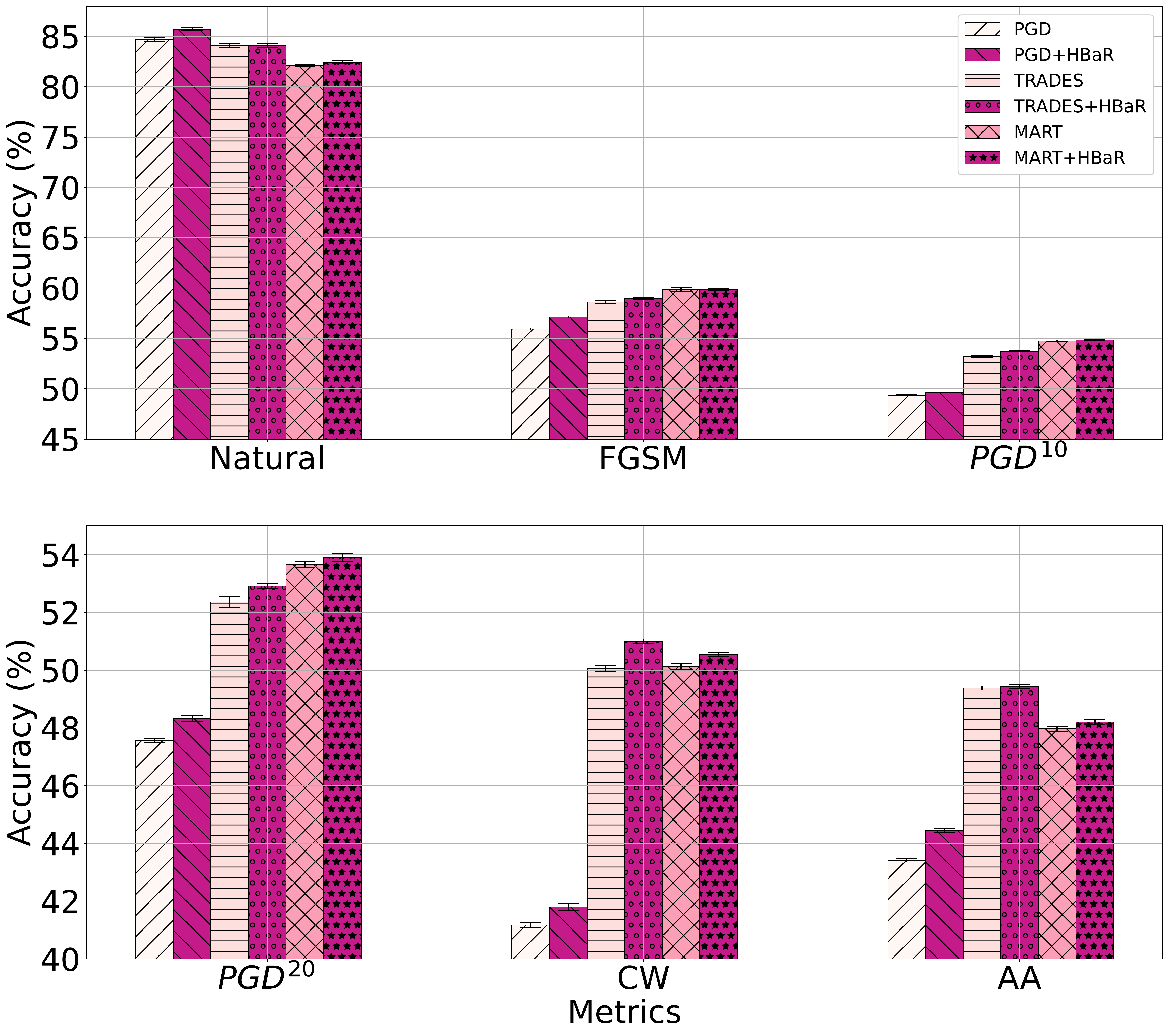} &
   \includegraphics[width=0.49\columnwidth]{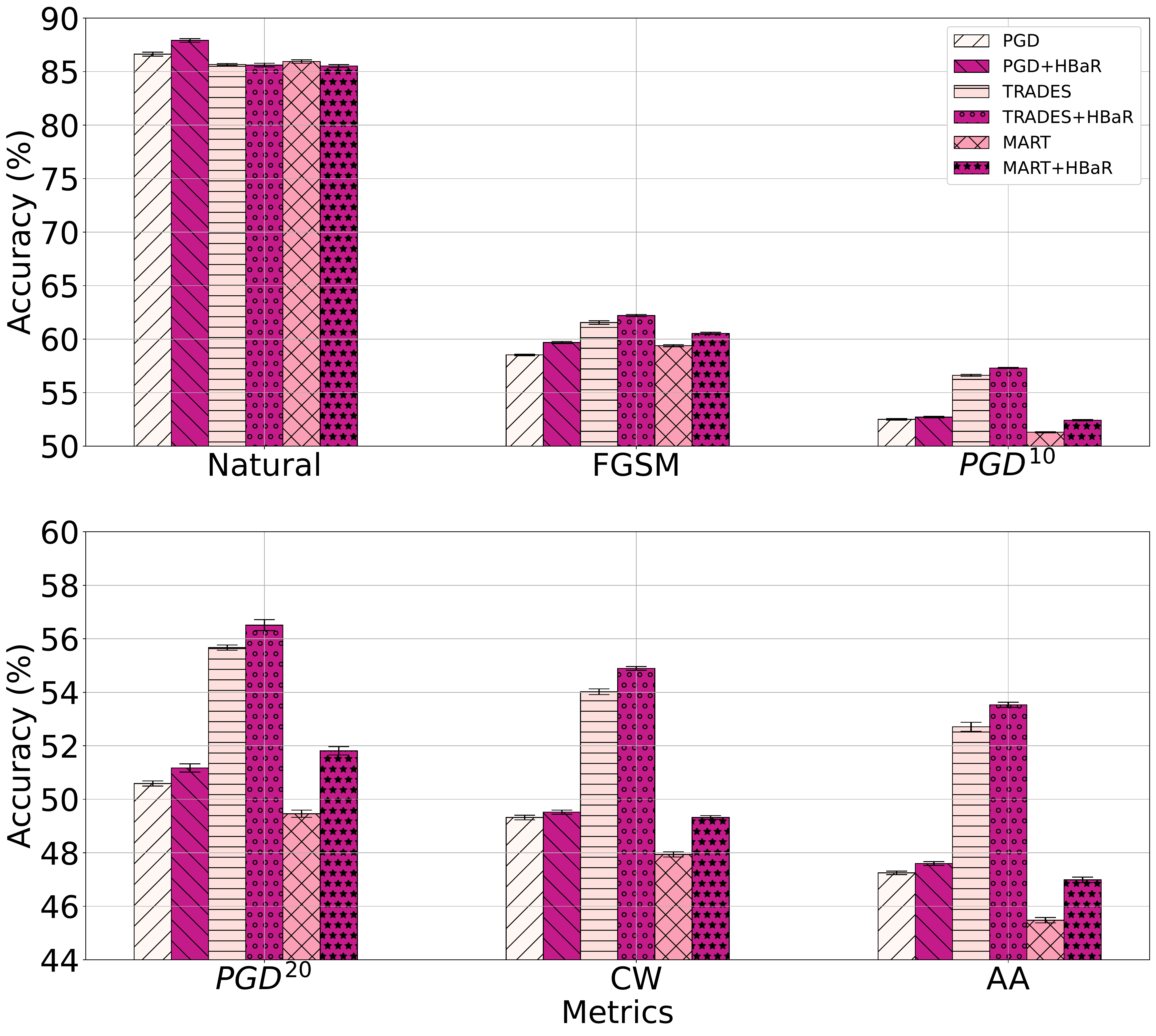} \\
   (c) CIFAR-10 by ResNet-18 & (d) CIFAR-10 by WideResNet-28-10 \\
\end{tabular}
	\caption{Error bar of natural test accuracy (in \%) and adversarial robustness ((in \%) on FGSM, PGD, CW, and AA attacked test examples) on MNIST by LeNet, CIFAR-100 by WideResNet-28-10, CIFAR-10 by ResNet-18 and WideResNet-28-10 of adversarial learning baselines and combining HBaR with each correspondingly.}
	\label{fig:hsic-adv-variance}
\end{figure*}

\begin{table*}[!t]
    \centering
    \setlength{\extrarowheight}{.2em}
    \setlength{\tabcolsep}{3pt}
    \small
    \caption{\textbf{MNIST by LeNet}: Mean and Standard deviation of natural test accuracy (in \%) and adversarial robustness ((in \%) on FGSM, PGD, CW, and AA attacked test examples) of adversarial learning baselines and combining HBaR with each correspondingly.}
    \label{tab:hsic-adv-variance-mnist}
    \resizebox{0.9\textwidth}{!}{
    \begin{tabular}{||c || c c c c c c||}
        \hline
        \multirow{2}{*}{Methods} & \multicolumn{6}{c||}{MNIST by LeNet} \\
        \cline{2-7}
        & Natural & FGSM & PGD$^{20}$ & PGD$^{40}$ & CW & AA \\
        \hline
        \hline
         PGD            & 98.40 $\pm$ 0.018 & 93.44 $\pm$ 0.177 & 94.56 $\pm$ 0.079 & 89.63 $\pm$ 0.117 & 91.20 $\pm$ 0.097 & 86.62 $\pm$ 0.166 \\
        \hb + PGD       & \textbf{98.66} $\pm$ 0.026 & \textbf{96.02} $\pm$ 0.161 & \textbf{96.44}$\pm$0.030 & \textbf{94.35}$\pm$0.130 & \textbf{95.10}$\pm$0.106 & \textbf{91.57}$\pm0.123$ \\
        \hline
        \hline
        TRADES          & \textbf{97.64}$\pm$0.017 & 94.73$\pm$0.196 & 95.05$\pm$0.006 & 93.27$\pm$0.088 & 93.05$\pm$0.025 & 89.66$\pm$0.085\\
        
        \hb + TRADES    & \textbf{97.64}$\pm$0.030 & \textbf{95.23}$\pm$0.106 & \textbf{95.17}$\pm$0.023 & \textbf{93.49}$\pm$0.147 & \textbf{93.47}$\pm$0.089 & \textbf{89.99}$\pm$0.155\\
        \hline
        \hline
        MART            & \textbf{98.29}$\pm$0.059 & 95.57$\pm$0.113 & 95.23$\pm$0.144 & 93.55$\pm$0.018 & 93.45$\pm$0.077 & 88.36$\pm$0.179 \\
        \hb + MART      & 98.23$\pm$0.054 & \textbf{96.09}$\pm$0.074 & \textbf{96.08}$\pm$0.035 & \textbf{94.64}$\pm$0.125 & \textbf{94.62}$\pm$0.06 & \textbf{89.99}$\pm$0.13 \\
        \hline
    \end{tabular}}
\end{table*}

\begin{table*}[!t]
    \centering
    \setlength{\extrarowheight}{.2em}
    \setlength{\tabcolsep}{3pt}
    \small
    \caption{\textbf{CIFAR-10 by ResNet-18}: Mean and Standard deviation of natural test accuracy (in \%) and adversarial robustness ((in \%) on FGSM, PGD, CW, and AA attacked test examples) of adversarial learning baselines and combining HBaR with each correspondingly.}
    \label{tab:hsic-adv-variance-cifar10}
    \resizebox{0.9\textwidth}{!}{
    \begin{tabular}{||c || c c c c c c ||}
        \hline
        \multirow{2}{*}{Methods} & \multicolumn{6}{c||}{CIFAR-10 by ResNet-18} \\
        \cline{2-7}
        & Natural & FGSM & PGD$^{10}$ & PGD$^{20}$ & CW & AA \\
        \hline
        \hline
         PGD            & 84.71$\pm$0.16 & 55.95$\pm$0.097 & 49.37$\pm$0.075 & 47.54$\pm$0.080 & 41.17$\pm$0.086 & 43.42$\pm$0.064 \\
         
        \hb + PGD       & \textbf{85.73}$\pm$0.166 & \textbf{57.13}$\pm$0.099 & \textbf{49.63}$\pm$0.058 & \textbf{48.32}$\pm$0.103 & \textbf{41.80}$\pm$0.116 & \textbf{44.46}$\pm$0.169 \\
        \hline
        \hline
        
        TRADES          & 84.07$\pm$0.201 & 58.63$\pm$0.167 & 53.21$\pm$0.118 & 52.36$\pm$0.189 & 50.07$\pm$0.106 & 49.38$\pm$0.069 \\
        
        \hb + TRADES    & \textbf{84.10}$\pm$0.104 & \textbf{58.97}$\pm$0.093 & \textbf{53.76}$\pm$0.080 & \textbf{52.92}$\pm$0.175 & \textbf{51.00}$\pm$0.085 & \textbf{49.43}$\pm$0.064 \\
        \hline
        \hline
        MART            & 82.15$\pm$0.117 & 59.85$\pm$0.154 & 54.75$\pm$0.089 & 53.67$\pm$0.088 & 50.12$\pm$0.106 & 47.97$\pm$0.156 \\
        \hb + MART      & \textbf{82.44}$\pm$0.156 & \textbf{59.86}$\pm$0.132 & \textbf{54.84}$\pm$0.051 & \textbf{53.89}$\pm$0.135 & \textbf{50.53}$\pm$0.069 & \textbf{48.21}$\pm$0.100 \\
        \hline
    \end{tabular}}
    \vspace{50pt}
\end{table*}

\begin{table*}[!t]
    \centering
    \setlength{\extrarowheight}{.2em}
    \setlength{\tabcolsep}{3pt}
    \small
    \caption{\textbf{CIFAR-10 by WideResNet-28-10}: Mean and Standard deviation of natural test accuracy (in \%) and adversarial robustness ((in \%) on FGSM, PGD, CW, and AA attacked test examples) of adversarial learning baselines and combining HBaR with each correspondingly.}
    \label{tab:hsic-adv-variance-cifar10-wrn}
    \resizebox{0.9\textwidth}{!}{
    \begin{tabular}{||c || c c c c c c ||}
        \hline
        \multirow{2}{*}{Methods} & \multicolumn{6}{c||}{CIFAR-10 by WideResNet-28-10} \\
        \cline{2-7}
        & Natural & FGSM & PGD$^{10}$ & PGD$^{20}$ & CW & AA \\
        \hline
        \hline
         PGD            & 86.63$\pm$0.186 & 58.53$\pm$0.073 & 52.21$\pm$0.084 & 50.59$\pm$0.096 & 49.32$\pm$0.089 & 47.25$\pm$0.124 \\
         
        \hb + PGD       & \textbf{87.91}$\pm$0.102 & \textbf{59.69}$\pm$0.097 & \textbf{52.72}$\pm$0.081 & \textbf{51.17}$\pm$0.152 & \textbf{49.52}$\pm$0.174 & \textbf{47.60}$\pm$0.131 \\
        \hline
        \hline
        
        TRADES          & \textbf{85.66}$\pm$0.103 & 61.55$\pm$0.134 & 56.62$\pm$0.097 & 55.67$\pm$0.098 & 54.02$\pm$0.106 & 52.71$\pm$0.169 \\
        
        \hb + TRADES    & 85.61$\pm$0.0133 & \textbf{62.20}$\pm$0.102 & \textbf{57.30}$\pm$0.059 & \textbf{56.51}$\pm$0.136 & \textbf{54.89}$\pm$0.098 & \textbf{53.53}$\pm$0.127 \\
        \hline
        \hline
        MART            & \textbf{85.94}$\pm$0.156 & 59.39$\pm$0.109 & 51.30$\pm$0.052 & 49.46$\pm$0.136 & 47.94$\pm$0.098 & 45.48$\pm$0.100 \\
        \hb + MART      & 85.52$\pm$0.136 & \textbf{60.54}$\pm$0.071 & \textbf{53.42}$\pm$0.142 & \textbf{51.81}$\pm$0.177 & \textbf{49.32}$\pm$0.131 & \textbf{46.99}$\pm$0.137 \\
        \hline
    \end{tabular}}
\end{table*}

\begin{table*}[!t]
    \centering
    \setlength{\extrarowheight}{.2em}
    \setlength{\tabcolsep}{3pt}
    \small
    \caption{\textbf{CIFAR-100 by WideResNet-28-10}: Mean and Standard deviation of natural test accuracy (in \%) and adversarial robustness ((in \%) on FGSM, PGD, CW, and AA attacked test examples) of adversarial learning baselines and combining HBaR with each correspondingly.}
    \label{tab:hsic-adv-variance-cifar100}
    \resizebox{0.9\textwidth}{!}{
    \begin{tabular}{||c || c c c c c c||}
        \hline
        \multirow{2}{*}{Methods} & \multicolumn{6}{c||}{CIFAR-100 by WideResNet-28-10} \\
        \cline{2-7}
        & Natural & FGSM & PGD$^{20}$ & PGD$^{40}$ & CW & AA \\
        \hline
        \hline
         PGD            & 59.91$\pm$0.116 & 29.85$\pm$0.117 & 26.05$\pm$0.106 & 25.38$\pm$0.129 & 22.28$\pm$0.079 & 20.91$\pm$0.133 \\
        \hb + PGD       & \textbf{63.84}$\pm$0.105 & \textbf{31.59}$\pm$0.054 & \textbf{27.90}$\pm$0.030 & \textbf{27.21}$\pm$0.025 & \textbf{23.23}$\pm$0.088 & \textbf{21.61}$\pm$0.061 \\
        \hline
        \hline
        TRADES          & 60.29$\pm$0.122 & 34.19$\pm$0.132 & 31.32$\pm$0.134 & 30.96$\pm$0.135 & 28.20$\pm$0.097 & 26.91$\pm$0.172 \\
        
        \hb + TRADES    & \textbf{60.55}$\pm$0.065 & \textbf{34.57}$\pm$0.068 & \textbf{31.96}$\pm$0.067 & \textbf{31.57}$\pm$0.079 & \textbf{28.72}$\pm$0.071 & \textbf{27.46}$\pm$0.098\\
        \hline
        \hline
        MART            & 58.42$\pm$0.164 & 32.94$\pm$0.160 & 29.17$\pm$0.166 & 28.19$\pm$0.252 & 27.31$\pm$0.096 & 25.09$\pm$0.179 \\
        \hb + MART      & \textbf{58.93}$\pm$0.102 & \textbf{33.49}$\pm$0.144 & \textbf{30.72}$\pm$0.130 & \textbf{30.16}$\pm$0.133 & \textbf{28.89}$\pm$0.118 & \textbf{25.21}$\pm$0.111 \\
        \hline
    \end{tabular}}
\end{table*}

\newpage
\section{Error Bar for Combining \hb with Adversarial Examples}\label{sec:supp-errorbar}
We show how \hb can be used to improve robustness when used as a regularizer, as described in \Cref{sec:combine-hb}, along with state-of-the-art adversarial learning methods. We run each experiment by five times. \fcom{Figure \ref{fig:hsic-adv-variance} illustrates mean and standard deviation of the natural test accuracy and adversarial robustness against various attacks on CIFAR-10 by ResNet-18 and WideResNet-28-10. Table \ref{tab:hsic-adv-variance-mnist}, \ref{tab:hsic-adv-variance-cifar10}, \ref{tab:hsic-adv-variance-cifar10-wrn}, and \ref{tab:hsic-adv-variance-cifar100} show the detailed standard deviation.} Combined with the adversarial training baselines, \hb consistently improves adversarial robustness against all types of attacks with small variance.

\newpage
\section{Limitations} \label{sec:supp-limit}
\com{
One limitation of our method is that the robustness gain, though beating other IB-based methods, is modest when training with only natural examples. However, the potential of getting adversarial robustness \emph{without} adversarial training is interesting and worth further exploration in the future. Another limitation of our method, as well as many proposed adversarial defense methods, is the uncertain performance to new attack methods. Although we have established concrete theories and conducted comprehensive experiments, there is no guarantee that our method is able to handle novel, well-designed attacks. Finally, in our theoretical analysis in Section~\ref{sec:HBAR_theorem}, we have made several assumptions for Theorem~\ref{thm:new_theorem}. While Assumptions~\ref{asm:cont} and~\ref{asm:universal} hold in practice, the distribution of input feature is not guaranteed to be standard Gaussian. Although the empirical evaluation supports the correctness of the theorem, we admit that the claim is not general enough. We aim to proof a more general version of Theorem~\ref{thm:new_theorem} in the future, hopefully agnostic to input distributions.
We will keep track of the advances in the adversarial robustness field and further improve our work correspondingly.}

\section{Potential Societal Negative Impact} \label{sec:supp-impact}
\com{
Although \hb has great potential as a general strategy to enhance the robustness for various machine learning systems, we still need to be aware of the potential negative societal impacts it might result in. For example, over-confidence in the \emph{adversarially-robust} models produced by \hb as well as other defense methods may lead to overlooking their potential failure on newly-invented attack methods; this should be taken into account in safety-critical applications like healthcare~\citep{adv_health} or security~\citep{adv_surveillance}. Another example is that, one might get insights from the theoretical analysis of our method to design stronger adversarial attacks. These attacks, if fall into the wrong hands, might cause severe societal problems. Thus, we encourage our machine learning community to further explore this field and be judicious to avoid misunderstanding or misusing of our method. Moreover, we propose to establish more reliable adversarial robustness checking routines for machine learning models deployed in safety-critical applications. For example, we should test these models with the latest adversarial attacks and make corresponding updates to them annually.
}

\end{document}